\documentclass[lettersize,journal]{IEEEtran}

\usepackage{amsthm}
\usepackage{amsmath}    
\usepackage{amssymb}    
\usepackage{bm}         
\usepackage{cite}       
\usepackage{algorithm}  
\usepackage{algorithmic}
\usepackage{graphicx}   

\newtheorem{theorem}{Theorem}      
\newtheorem{remark}{Remark}  

\begin{document}
	
	\title{Adaptive Genetic Selection based Pinning Control with Asymmetric Coupling for Multi-Network Heterogeneous Vehicular Systems}
	
	\author{Weian Guo, Ruizhi Sha, Li Li, Lun Zhang and Dongyang Li
		\thanks{This work was supported in part by the National Natural Science Foundation of China under Grant 62273263, 72171172, 92367101 and 71771176; the Aeronautical Science Foundation of China under Grant 2023Z066038001; the National Natural Science Foundation of China Basic Science Research Center Program under Grant 62088101;Municipal Science and Technology Major Project (2022-5-YB-09); Natural Science Foundation of Shanghai under Grant Number 23ZR1465400. (Corresponding author: Li Li). Thanks for Mr. Wanli CAI for providing experiment environments to debug the codes.}
		\thanks{Weian Guo and Dongyang Li are with Sino-German College of Applied Sciences, Tongji University, Shanghai, China (email:\{guoweian, lidongyang0412\}@163.com). Lun Zhang is with School of Transportation, Tongji University, Shanghai, China (email: lun$\_$zhang@tongji.edu.cn).  Li Li and Ruizhi Sha are with the Department of Electronics and Information Sciences, Tongji University, Shanghai, China (email: \{lili, rzsha\}@tongji.edu.cn).}
	}
	
	\maketitle
	
	\begin{abstract}
		To alleviate computational load on RSUs and cloud platforms, reduce communication bandwidth requirements, and provide a more stable vehicular network service, this paper proposes an optimized pinning control approach for heterogeneous multi-network vehicular ad-hoc networks (VANETs). In such networks, vehicles participate in multiple task-specific networks with asymmetric coupling and dynamic topologies. We first establish a rigorous theoretical foundation by proving the stability of pinning control strategies under both single and multi-network conditions, deriving sufficient stability conditions using Lyapunov theory and linear matrix inequalities (LMIs). Building on this theoretical groundwork, we propose an adaptive genetic algorithm tailored to select optimal pinning nodes, effectively balancing LMI constraints while prioritizing overlapping nodes to enhance control efficiency. Extensive simulations across various network scales demonstrate that our approach achieves rapid consensus with a reduced number of control nodes, particularly when leveraging network overlaps. This work provides a comprehensive solution for efficient control node selection in complex vehicular networks, offering practical implications for deploying large-scale intelligent transportation systems.
	\end{abstract}	
	
	\section{Introduction}
	\label{sec:intro}
	\IEEEPARstart{V}{ehicular} Ad-hoc Networks (VANETs) are foundational to intelligent transportation systems (ITS), providing a framework for real-time vehicle-to-vehicle (V2V) and vehicle-to-infrastructure (V2I) communication essential for traffic coordination and road safety \cite{pierpaolo_salvo_05076567, haixia_peng_7422716c}. As the scale and complexity of VANETs increase, so do the computational and communication demands on roadside units (RSUs) and cloud platforms, which serve as critical infrastructure for processing and distributing control information \cite{guo2025RSU}. To alleviate these demands and enhance VANET stability, efficient control methods that reduce reliance on extensive infrastructure are crucial for practical ITS deployment \cite{wenjun_xu_36c84a48}.
	
	In VANETs, achieving network controllability—where a subset of vehicles can steer the entire network to desired states—is essential for adaptive responses to dynamic traffic conditions and diverse coordination needs \cite{nelson_cardona_6a6a5a0f}. VANETs operate in diverse environments with varying network topologies tailored to specific communication and control demands. For instance, urban areas with dense intersections may use mesh or cluster-based topologies to handle high traffic volumes and complex road layouts, while highways might adopt chain topologies to support convoy formations. Furthermore, communication links in VANETs often exhibit asymmetric coupling, where leading vehicles influence the behavior of following vehicles, but the reverse influence is minimal or absent. In emergency scenarios, such as ambulances requesting right-of-way, asymmetric interactions become even more pronounced as critical messages broadcast without expecting reciprocal communication \cite{sok_ian_sou_fabc299a, hieu_nguyen_507b4c14, justin_j__boutilier_28ca4f44}. These factors necessitate specialized control strategies to ensure network cohesion and achieve control objectives effectively.
	
	\begin{figure*}[ht]
		\centering
		\includegraphics[width=\textwidth]{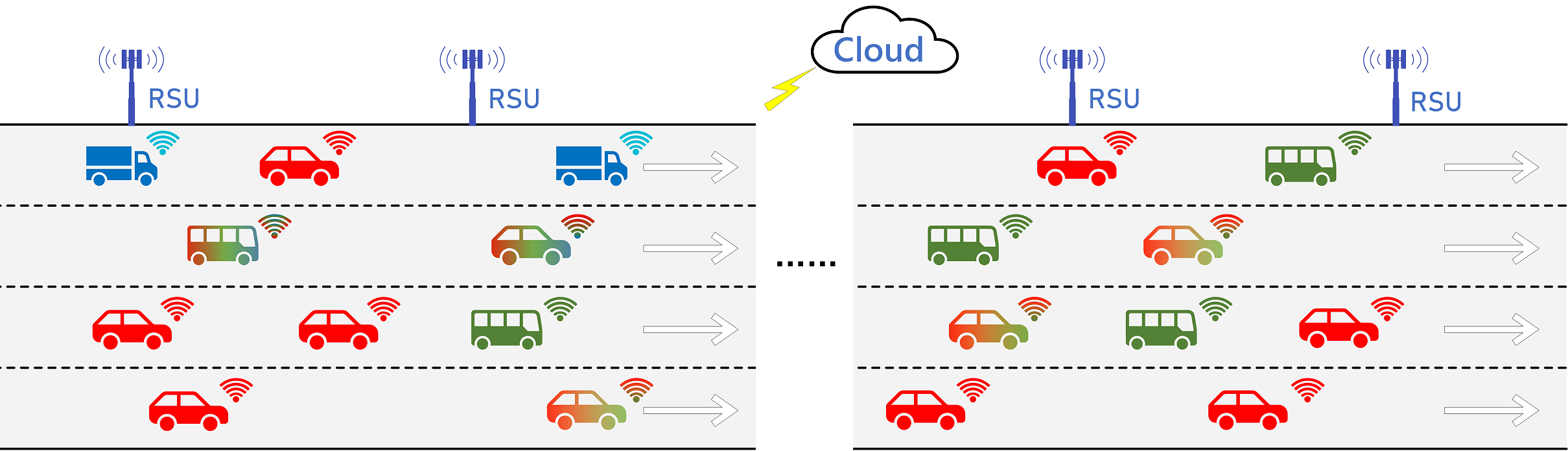}
		\caption{Illustration of a multi-network vehicular environment. The red icons represent private cars, green icons represent buses, and blue icons represent freight trucks. Multi-color vehicles simultaneously belong to different vehicular networks, serving as overlapping nodes with unique operational requirements and behaviors tailored to each network.}
		\label{fig:traffic_pin}
	\end{figure*}
	
	To meet these challenges, achieving network-wide stability by controlling a strategically selected subset of vehicles has emerged as a promising approach, effectively reducing reliance on RSUs and cloud computing resources \cite{milad_pooladsanj_b102bf01}. However, as illustrated in Fig. \ref{fig:traffic_pin}, VANETs often encompass multiple task-specific networks, each with distinct objectives, such as congestion management, collision avoidance, or coordinated route planning. Vehicles that participate in multiple networks—depicted by multi-colored icons—serve as overlapping nodes, introducing unique control opportunities and challenges. Leveraging these overlapping nodes within a pinning control framework can enhance control efficiency and reduce control costs, though it requires balancing the objectives of each network to prevent conflicts, particularly in safety-critical scenarios.
	
	Further complicating VANET control is the diversity of vehicle types, such as private cars, buses, and freight trucks, each with distinct operational requirements \cite{azzedine_boukerche_6fa03cb5}. For example, private cars often exhibit adaptive driving patterns, buses prioritize route consistency and passenger safety, and freight trucks frequently travel in convoys to optimize logistical efficiency. RSUs, positioned along the road, play a vital role in V2I communication by facilitating data exchange between vehicles and the cloud, which acts as a centralized platform for control and data processing. In infrastructure-limited scenarios, the deployment of a pinning control strategy becomes essential for achieving efficient, network-wide control, thereby reducing the computational load and dependency on cloud resources.
	
	In autonomous driving environments, the state demands of vehicles—such as speed, inter-vehicle distance, braking status, and desired trajectory—further complicate control node selection \cite{jinglai_shen_360e3888, kanika_sharma_5082d3c4}. Multi-colored vehicles in Fig. \ref{fig:traffic_pin} represent those with diverse state requirements, often forming sub-networks within the larger vehicular network. Applying direct control to each vehicle through RSUs and cloud computing would impose substantial computational and communication overhead. By strategically selecting and controlling specific vehicles (pinning control nodes), state influence can propagate through overlapping nodes, promoting consensus across sub-networks and achieving efficient network-wide control.
	
	This paper proposes an adaptive genetic algorithm-based pinning control framework aimed at enhancing network controllability in VANETs with asymmetric coupling and overlapping nodes. We establish a theoretical foundation for pinning control in both single and multi-network scenarios, proving the stability of the control strategy using Lyapunov theory and linear matrix inequalities (LMIs). Building on this analysis, we design a novel genetic algorithm to optimize the selection of control nodes, focusing on minimizing the number of nodes required while ensuring overall network stability. The proposed approach reduces computational and control load across the network, thus alleviating dependency on external infrastructure. The contributions of this paper are summarized as follows:
	\begin{itemize}
		\item We conduct a comprehensive controllability analysis, establishing theoretical conditions for pinning control stability in VANETs with asymmetric coupling and overlapping nodes, thereby supporting the practical deployment of robust ITS.
		\item We introduce an adaptive genetic algorithm to optimize control node selection, balancing LMI constraints and minimizing control nodes, effectively reducing control costs and infrastructure dependency.
		\item We perform extensive simulations across diverse network topologies and dynamic vehicular scenarios, demonstrating the efficiency and robustness of the proposed framework.
	\end{itemize}
	
	The remainder of this paper is organized as follows. Section~\ref{sec:related} reviews relevant work on VANET control and network topology management. Section~\ref{sec:model} presents the system model and outlines the genetic algorithm-based optimization approach. Section~\ref{sec:ga} details the proposed genetic selection strategy for control nodes. Section~\ref{sec:sim} discusses the simulation setup and results, with a focus on the impact of different network topologies on controllability. Finally, Section~\ref{sec:con} concludes the paper and suggests potential directions for future research.
	
	\section{Related Work}
	\label{sec:related}
	Vehicular Ad-hoc Networks (VANETs) have been the subject of extensive academic inquiry due to their potential in enabling intelligent transportation systems. However, the high mobility, frequent topology changes, and complex network interactions inherent in VANETs present significant challenges for stable and efficient network control. Research has been conducted on a variety of VANET-related topics, such as communication protocols, network topology management, control strategies, and optimization techniques for control node selection \cite{savas_konur_9b43c5bb, deepak_gupta_57289d84}.
	
	In this section, we provide an overview of the related work on VANET topologies, control methodologies, and optimization techniques for control node selection. To address these challenges, adaptive and decentralized control strategies have been proposed to maintain VANET stability and improve coordination \cite{nelson_cardona_6a6a5a0f, chander_bhanu_8bed2b88, deepak_gupta_57289d84, sergio_luis_o__b__correia_e992a3a9}. Consensus-based control approaches have been extensively studied for achieving cooperative behavior among vehicles. These methods aim to synchronize the state of vehicles, such as speed and acceleration, to ensure smooth traffic flow and prevent collisions \cite{ananto_tri_sasongko_9f7f9e15}. Achieving consensus in VANET environments involves various strategies such as event-triggered consensus and distributed model predictive control \cite{vanshri_deshpande_3d4579cb, stefania_santini_ec928cea, savas_konur_9b43c5bb}. However, achieving consensus in dynamic VANET environments is challenging due to the unpredictable nature of network connectivity and the asymmetric interactions among vehicles \cite{nelson_cardona_6a6a5a0f, chander_bhanu_8bed2b88}. Recent research has introduced adaptive consensus protocols that consider time-varying topology changes, while others have proposed reinforcement learning-based approaches to improve consensus convergence in highly dynamic scenarios \cite{arunselvan_ramaswamy_eca2d453, maryam_sharifi_0d4ad56b, morteza_mirzaei_7b1ad9b7}. Despite their effectiveness, these methods often require the entire network to participate in control, which can be resource-intensive. 
	
	Pinning control, a specialized form of consensus control, has been introduced to reduce the control effort by applying control inputs to a selected subset of nodes \cite{cameron_nowzari_66429cd7, wei_liu_37ced799}. Pinning control has demonstrated significant potential in enhancing network controllability, particularly in large-scale multi-agent systems \cite{gang_chen_6c711c7e, tong_zhou_05fff1c1}. In the context of VANETs, pinning control can help reduce resource consumption while maintaining system stability, especially in the face of dynamic topologies \cite{gang_chen_6c711c7e}. Existing studies have explored the impact of pinning node selection on the overall controllability and robustness of the network \cite{deepak_gupta_57289d84}. Furthermore, adaptive pinning strategies have been developed to adjust the control nodes in response to changes in the network topology, ensuring better resilience against link failures \cite{xiaofan_wang_05024a6b}. However, the applicability of pinning control in VANETs, particularly under real-world constraints such as high mobility and frequent disconnections, remains underexplored. 
	
	Optimization of control node selection is another area that has received considerable attention. The concept of using leader-follower models in VANETs has been explored to minimize control efforts by selecting a subset of vehicles as leaders while ensuring that the entire system remains controllable \cite{rusheng_zhang_049f65c4, christy_jackson_joshua_78aaf324, ahmad_mostafa_faf4e367}. Various methods have been employed for leader selection, including linear programming, heuristic algorithms, and metaheuristic approaches \cite{rusheng_zhang_049f65c4, rasmeet_singh_bali_d6e9b9c4}. Genetic algorithms (GAs) have shown promise due to their ability to handle complex, non-linear optimization problems and explore a large solution space efficiently. Recent advancements include hybrid genetic algorithms that combine GAs with other optimization techniques to improve performance \cite{adam_s_owik_761f3fb4}. Additionally, particle swarm optimization (PSO) and ant colony optimization (ACO) have been applied to leader selection problems in VANETs with promising results \cite{christy_jackson_joshua_78aaf324, forough_goudarzi_a6e18e54, bernab__dorronsoro_85077f67}. However, most existing optimization approaches do not explicitly consider the dynamic nature of VANETs, such as frequent topology changes and asymmetric coupling, which limits their applicability in real-world scenarios. Recent work has begun to address these limitations by incorporating adaptive optimization frameworks that respond to changes in network conditions, ensuring more robust control performance \cite{shankar_a__deka_e25bb81a, linghe_kong_d528404c}.
	
	Although considerable progress has been made in VANET control strategies, existing studies have not fully addressed the challenges posed by high mobility and frequent disconnections inherent in VANETs. Furthermore, issues such as asymmetric coupling between vehicle nodes and the extensive network-wide participation required by many control strategies are often overlooked \cite{saleh_yousefi_94da103c, jitendra_bhatia_27792768}. These limitations hinder the applicability of many control and optimization techniques in real-world VANET scenarios, where dynamic topology changes and varying network conditions are prevalent \cite{tanjida_kabir_b662b9e2, xin_yang_1152c802, hamid_barkouk_4fe2b61d}.
	
	In this paper, we build upon these prior works by proposing an adaptive genetic algorithm for optimizing control node selection in VANETs. Unlike previous studies that focus on static topologies or simplified network conditions, our approach explicitly addresses dynamic topologies and asymmetric coupling, ensuring system controllability under realistic VANET conditions. By applying control inputs to a selected subset of nodes through pinning control, our method reduces overall control effort and resource consumption while maintaining system stability. Additionally, we perform a detailed controllability analysis to validate the effectiveness of our control strategies under the complex, dynamic conditions typical of VANETs.

	\section{Problem Description and Analysis}
	\label{sec:model}	
	\subsection{Pinning Consensus Problem in a Single Vehicular Network System}
	
	In this subsection, we study the pinning consensus problem in a first-order continuous-time vehicular ad-hoc network (VANET) system using adaptive and pinning control strategies. We analyze the conditions necessary for system convergence and propose corresponding pinning control methods.
	
	Consider a continuous-time vehicular network system consisting of $N$ vehicles. The state of each vehicle is described by physical quantities such as position and velocity. The dynamic equation for each vehicle is:
	\begin{equation}
		\label{eqn:dynamics}
		\dot{x}_i(t) = C \sum_{j=1}^N G_{ij} \Gamma [x_j(t) - x_i(t)] - c d_i \Gamma [x_i(t) - x^*],
	\end{equation}
	where:
	\begin{itemize}
		\item $x_i(t) \in \mathbb{R}^m$ represents the state vector of vehicle $i$ at time $t$, including position, velocity, etc.;
		\item $C > 0$ is the coupling strength that determines the influence between connected vehicles;
		\item $\Gamma \in \mathbb{R}^{m \times m}$ is the internal coupling matrix, satisfying $\Gamma \succ 0$ (i.e., $\Gamma$ is positive definite);
		\item $G_{ij}$ are the elements of the asymmetric adjacency matrix $G \in \mathbb{R}^{N \times N}$, describing the communication links between vehicles;
		\item $c > 0$ is the pinning control gain constant that adjusts the strength of the pinning control;
		\item $d_i \in \{0, 1\}$ indicates whether vehicle $i$ is pinned ($d_i = 1$);
		\item $x^* \in \mathbb{R}^m$ is the target consensus state to which all vehicles should converge.
	\end{itemize}
	
	To simplify the analysis and focus on the core aspects of pinning control, we consider the case without communication delays, i.e., $\tau_{ij} = 0$. The system dynamics then reduce to:
	\begin{equation}
		\label{eqn:base_model}
		\dot{x}_i(t) = C \sum_{j=1}^N G_{ij} \Gamma [x_j(t) - x_i(t)] - c d_i \Gamma [x_i(t) - x^*].
	\end{equation}
	
	Define the error:
	\begin{equation}
		\label{eqn:error}
		e_i(t) = x_i(t) - x^*, \quad i = 1, 2, \dots, N.
	\end{equation}
	The condition for achieving pinning consensus is $\lim_{t \to \infty} \|e_i(t)\| = 0$ for all $i$. Substituting \eqref{eqn:error} into \eqref{eqn:base_model}, we obtain:
	\begin{equation}
		\label{eqn:error_full}
		\dot{e}_i(t) = C \sum_{j=1}^N G_{ij} \Gamma [e_j(t) - e_i(t)] - c d_i \Gamma e_i(t).
	\end{equation}
	
	Using the Laplacian matrix $L$, we define the degree matrix $D$, whose diagonal elements are $D_{ii} = \sum_{j=1}^N G_{ij}$. The Laplacian matrix $L$ is defined as:
	\begin{equation}
		\label{eqn:Laplacian_matrix}
		L_{ij} = 
		\begin{cases} 
			D_{ii}, & \text{if } i = j, \\
			- G_{ij}, & \text{if } i \neq j.
		\end{cases}
	\end{equation}
	
	Thus, \eqref{eqn:error_full} can be rewritten as:
	\begin{equation}
		\label{eqn:error_variants}
		\dot{e}_i(t) = -C \sum_{j=1}^N L_{ij} \Gamma e_j(t) - c d_i \Gamma e_i(t).
	\end{equation}
	
	Expressing the above equation in matrix form:
	\begin{equation}
		\label{eqn:error_matrix}
		\dot{e}(t) = -C (L \otimes \Gamma) e(t) - c (\hat{D} \otimes \Gamma) e(t),
	\end{equation}
	where $e(t) = [e_1^\mathrm{T}(t), e_2^\mathrm{T}(t), \dots, e_N^\mathrm{T}(t)]^\mathrm{T} \in \mathbb{R}^{Nm}$, and $\hat{D} = \operatorname{diag}(d_1, d_2, \dots, d_N)$.
	
	In VANETs, it is common to achieve pinning consensus through a properly designed control strategy, so that all vehicles converge to a common target state:
	\[
	\lim_{t \to \infty} \| x_i(t) - x^* \| = \lim_{t \to \infty} \| e_i(t) \| = 0, \quad i = 1, 2, \dots, N.
	\]
	
	\begin{theorem}
		\label{theorem:pinning_consensus}
		Consider the continuous-time vehicular network system \eqref{eqn:base_model} consisting of $N$ vehicles, where the adjacency matrix $G$ leads to a Laplacian matrix $L$ that is stable (i.e., all non-zero eigenvalues of $L$ have positive real parts). Let $Q = q I_m \succ 0$, where $q > 0$, and $\Gamma = I_m$. If there exists a sufficiently large constant $c > 0$ such that the following matrix inequality holds:
		\[
		2C L_s \otimes Q + 2c \hat{D} \otimes Q \geq \delta (I_N \otimes I_m),
		\]
		where $L_s = \frac{L + L^\mathrm{T}}{2}$ is the symmetric part of the Laplacian matrix $L$, and $\delta > 0$ is a given constant, then the system \eqref{eqn:base_model} can asymptotically achieve pinning consensus, i.e.,
		\[
		\lim_{t \to \infty} \| x_i(t) - x^* \| = 0, \quad i = 1, 2, \dots, N.
		\]
	\end{theorem}
	
	\begin{proof}
		Our goal is to prove that under the given conditions, the error vector $e(t) \in \mathbb{R}^{Nm}$, where $e_i(t) = x_i(t) - x^*$, converges to zero as $t \to \infty$. The error dynamics are:
		\begin{equation}
			\label{eqn:error_dynamics}
			\dot{e}(t) = -C (L \otimes \Gamma) e(t) - c (\hat{D} \otimes \Gamma) e(t).
		\end{equation}
		
		Since $L$ and $G$ may be asymmetric matrices, to apply Lyapunov methods, we decompose $L$ into its symmetric and skew-symmetric parts:
		\[
		L = L_s + L_a,
		\]
		where
		\[
		L_s = \frac{L + L^\mathrm{T}}{2}, \quad L_a = \frac{L - L^\mathrm{T}}{2}.
		\]
		
		Choose the Lyapunov function:
		\[
		V(t) = e^\mathrm{T}(t) (I_N \otimes Q) e(t),
		\]
		where $Q = q I_m \succ 0$. Compute the derivative of $V(t)$:
		\begin{equation}
			\label{eqn:V_derivative}
			\dot{V}(t) = 2 e^\mathrm{T}(t) (I_N \otimes Q) \dot{e}(t).
		\end{equation}
		
		Substituting \eqref{eqn:error_dynamics} into \eqref{eqn:V_derivative}, we obtain:
		\[
		\dot{V}(t) = -2C e^\mathrm{T}(t) (L \otimes Q \Gamma) e(t) - 2c e^\mathrm{T}(t) (\hat{D} \otimes Q \Gamma) e(t).
		\]
		
		Since $\Gamma = I_m$ and $Q$ is a scalar matrix, $Q \Gamma = Q$. Noting that $L_a$ is skew-symmetric and $Q$ is symmetric, we have:
		\[
		e^\mathrm{T}(t) (L_a \otimes Q) e(t) = 0.
		\]
		Therefore,
		\[
		\dot{V}(t) = -2C e^\mathrm{T}(t) (L_s \otimes Q) e(t) - 2c e^\mathrm{T}(t) (\hat{D} \otimes Q) e(t).
		\]
		
		Since $L_s$ is positive semi-definite (because the non-zero eigenvalues of $L$ have positive real parts), and $\hat{D}$ is a non-negative diagonal matrix, using the matrix inequality condition:
		\[
		2C L_s \otimes Q + 2c \hat{D} \otimes Q \geq \delta (I_N \otimes I_m),
		\]
		we obtain:
		\[
		\dot{V}(t) \leq -\delta e^\mathrm{T}(t) e(t) = -\delta \| e(t) \|^2.
		\]
		
		Therefore, $\dot{V}(t)$ is negative definite, and $V(t)$ is positive definite. By Lyapunov's stability theorem, the error vector $e(t)$ asymptotically converges to zero:
		\[
		\lim_{t \to \infty} e(t) = 0.
		\]
		Thus, the system achieves pinning consensus.
	\end{proof}
	
	\begin{remark}
		The control gain $c$ can be effectively determined using Linear Matrix Inequality (LMI) methods. If the LMI has no solution, one can increase the number of non-zero elements in the pinning matrix $\hat{D}$ to enhance pinning control, thereby obtaining greater control authority to satisfy the LMI. In the worst case, setting all vehicles as pinning nodes can achieve consensus.
	\end{remark}
	
	\subsection{Pinning Consensus in Multi-Network Heterogeneous Vehicular Systems}
	
	In this subsection, we extend the pinning consensus problem to a multi-network heterogeneous vehicular system with overlapping nodes and different target states. Each vehicular network may have its own dynamics and desired consensus state, and some vehicles may belong to multiple networks simultaneously.
	
	\subsubsection{System Model}
	
	Consider $K$ continuous-time vehicular networks, each denoted as $\mathcal{G}^{(k)} = (V^{(k)}, E^{(k)})$, where $k = 1, 2, \dots, K$. $V^{(k)}$ is the set of vehicles in network $k$, and $E^{(k)}$ is the set of communication links in network $k$. The overall set of vehicles is $V = \bigcup_{k=1}^K V^{(k)}$, and the set of overlapping vehicles is $V_{\text{overlap}} = \{ i \in V \mid \exists k_1 \neq k_2, i \in V^{(k_1)} \cap V^{(k_2)} \}$.
	
	Each vehicle $i \in V$ has a state vector $x_i(t) \in \mathbb{R}^m$, which is shared among all networks it belongs to, simulating the practical scenario where a vehicle's physical state (e.g., position, velocity) is shared in different network applications. The dynamics of vehicle $i$ are:
	\begin{equation}
		\label{eqn:multi_dynamics}
		\begin{aligned}
			\dot{x}_i(t) = \sum_{k \in \mathcal{K}_i} \Bigg[ & C^{(k)} \sum_{j \in \mathcal{N}_i^{(k)}} G_{ij}^{(k)} \Gamma^{(k)} [x_j(t) - x_i(t)] \\
			& - c^{(k)} d_i^{(k)} \Gamma^{(k)} [x_i(t) - x^{*(k)}] \Bigg],
		\end{aligned}
	\end{equation}
	where:
	\begin{itemize}
		\item $\mathcal{K}_i$ is the set of networks that vehicle $i$ belongs to;
		\item $\mathcal{N}_i^{(k)}$ is the set of neighbors of vehicle $i$ in network $k$;
		\item $C^{(k)} > 0$ is the coupling strength in network $k$;
		\item $\Gamma^{(k)} \in \mathbb{R}^{m \times m}$ is the internal coupling matrix for network $k$, satisfying $\Gamma^{(k)} \succ 0$;
		\item $G_{ij}^{(k)}$ are the elements of the adjacency matrix $G^{(k)}$ for network $k$;
		\item $c^{(k)} > 0$ is the pinning control gain for network $k$;
		\item $d_i^{(k)} \in \{0, 1\}$ indicates whether vehicle $i$ is pinned in network $k$;
		\item $x^{*(k)} \in \mathbb{R}^{m}$ is the target consensus state for network $k$.
	\end{itemize}
	
	\subsubsection{Problem Statement}
	
	Our objective is to design pinning control strategies $d_i^{(k)}$ and control gains $c^{(k)}$ such that each network converges to its respective target state $x^{*(k)}$, while minimizing the total number of pinned nodes, especially favoring overlapping nodes as pinning nodes.
	
	\subsubsection{Assumptions}
	
	To ensure the feasibility of the problem, we make the following assumptions:
	
	\begin{enumerate}
		\item The target states of different networks are consistent or can be aggregated compatibly at overlapping nodes, so that overlapping nodes can converge to a composite target state without conflict. Specifically, for overlapping nodes, we define a composite target state:
		\begin{equation}
			\label{eqn:composite_target}
			x_i^* = \frac{1}{|\mathcal{K}_i|} \sum_{k \in \mathcal{K}_i} x^{*(k)}, \quad \forall i \in V_{\text{overlap}}.
		\end{equation}
		\item The internal coupling matrices $\Gamma^{(k)}$ are identical across networks, or for simplicity, are multiples of the identity matrix, i.e., $\Gamma^{(k)} = \gamma^{(k)} I_m$, where $\gamma^{(k)} > 0$.
		\item The control gains $c^{(k)}$ can be appropriately chosen to balance the influence of different networks on overlapping nodes.
	\end{enumerate}
	
	\begin{theorem}
		\label{theorem:multi_pinning_consensus_modified}
		Consider the multi-network heterogeneous vehicular system described above. Suppose the internal coupling matrices are $\Gamma^{(k)} = \gamma^{(k)} I_m$, where $\gamma^{(k)} > 0$. Let $Q = q I_m \succ 0$, where $q > 0$. If there exist sufficiently large constants $c^{(k)} > 0$ and $\delta > 0$ such that the following matrix inequality holds:
		\begin{equation}
			\label{eqn:matrix_inequality}
			2 \sum_{k=1}^K \left[ C^{(k)} \gamma^{(k)} L_s^{(k)} \otimes Q + c^{(k)} \gamma^{(k)} \hat{D}^{(k)} \otimes Q \right] \geq \delta (I_N \otimes I_m),
		\end{equation}
		where $L_s^{(k)} = \frac{L^{(k)} + [L^{(k)}]^\mathrm{T}}{2}$ is the symmetric part of the Laplacian matrix $L^{(k)}$, and $\hat{D}^{(k)} = \operatorname{diag}(d_1^{(k)}, d_2^{(k)}, \dots, d_N^{(k)})$, then the system \eqref{eqn:multi_dynamics} can asymptotically achieve pinning consensus for all $i \in V$, i.e.,
		\begin{equation}
			\label{eqn:consensus_result}
			\lim_{t \to \infty} \| x_i(t) - x_i^* \| = 0, \quad \forall i \in V.
		\end{equation}
	\end{theorem}
	
	\begin{proof}
		Our goal is to show that under the given conditions, the error vector $e(t) = x(t) - x^*(t)$ satisfies $\lim_{t \to \infty} e(t) = 0$, where $x(t) = [x_1^\mathrm{T}(t), x_2^\mathrm{T}(t), \dots, x_N^\mathrm{T}(t)]^\mathrm{T}$, and $x^*(t) = [x_1^{*\mathrm{T}}(t), x_2^{*\mathrm{T}}(t), \dots, x_N^{*\mathrm{T}}(t)]^\mathrm{T}$.
		
		For each node $i \in V$, define the error:
		\begin{equation}
			\label{eqn:total_error}
			e_i(t) = x_i(t) - x_i^*(t).
		\end{equation}
		
		Substituting the error definition into the system dynamics \eqref{eqn:multi_dynamics}, we obtain the error dynamics:
		\begin{equation}
			\begin{aligned}
				\label{eqn:error_dynamics_initial}
				\dot{e}_i(t) = \sum_{k \in \mathcal{K}_i} \Bigg[ & -C^{(k)} \gamma^{(k)} \sum_{j \in \mathcal{N}_i^{(k)}} L_{ij}^{(k)} e_j(t) \\
				& - c^{(k)} \gamma^{(k)} d_i^{(k)} e_i(t) \Bigg],
			\end{aligned}
		\end{equation}
		where $L_{ij}^{(k)}$ are the elements of the Laplacian matrix $L^{(k)}$ for network $k$.
		
		Combine the error vectors of all nodes into a global error vector:
		\begin{equation}
			\label{eqn:global_error}
			e(t) = [e_1^\mathrm{T}(t), e_2^\mathrm{T}(t), \dots, e_N^\mathrm{T}(t)]^\mathrm{T}.
		\end{equation}
		
		Using the Kronecker product and properties of the Laplacian matrices, the global error dynamics can be expressed as:
		\begin{equation}
			\begin{aligned}
				\label{eqn:error_dynamics_matrix}
				\dot{e}(t) = -\sum_{k=1}^K \Big[ & C^{(k)} \gamma^{(k)} (L^{(k)} \otimes I_m) \\
				& + c^{(k)} \gamma^{(k)} (\hat{D}^{(k)} \otimes I_m) \Big] e(t).
			\end{aligned}
		\end{equation}
		
		Choose the Lyapunov function:
		\begin{equation}
			\label{eqn:lyapunov_function}
			V(t) = e^\mathrm{T}(t) (I_N \otimes Q) e(t),
		\end{equation}
		where $Q = q I_m \succ 0$.
		
		Compute the derivative of $V(t)$:
		\begin{equation}
			\label{eqn:lyapunov_derivative}
			\dot{V}(t) = 2 e^\mathrm{T}(t) (I_N \otimes Q) \dot{e}(t).
		\end{equation}
		
		Substituting the error dynamics from \eqref{eqn:error_dynamics_matrix} into \eqref{eqn:lyapunov_derivative}, we have:
		\begin{equation}
		\begin{aligned}
			\label{eqn:lyapunov_derivative_substituted}
			\dot{V}(t) = -2 e^\mathrm{T}(t) \sum_{k=1}^K \Big[ & C^{(k)} \gamma^{(k)} (L^{(k)} \otimes Q) \\
			& + c^{(k)} \gamma^{(k)} (\hat{D}^{(k)} \otimes Q) \Big] e(t).
		\end{aligned}
		\end{equation}

		Decompose each Laplacian matrix $L^{(k)}$ into its symmetric and skew-symmetric parts:
		\begin{equation}
			\label{eqn:laplacian_decomposition}
			L^{(k)} = L_s^{(k)} + L_a^{(k)},
		\end{equation}
		where:
		\begin{equation}
			\label{eqn:symmetric_part}
			L_s^{(k)} = \frac{L^{(k)} + [L^{(k)}]^\mathrm{T}}{2},\quad L_a^{(k)} = \frac{L^{(k)} - [L^{(k)}]^\mathrm{T}}{2}.
		\end{equation}
		
		Since $Q$ and $I_m$ are symmetric matrices, and $L_a^{(k)}$ is skew-symmetric, it follows that:
		\begin{equation}
			\label{eqn:skew_symmetric_zero}
			e^\mathrm{T}(t) (L_a^{(k)} \otimes Q) e(t) = 0.
		\end{equation}
		
		Therefore, the derivative of the Lyapunov function simplifies to:
		\begin{equation}
			\begin{aligned}
				\label{eqn:lyapunov_derivative_simplified}
				\dot{V}(t) = -2 e^\mathrm{T}(t) \sum_{k=1}^K \Big[ & C^{(k)} \gamma^{(k)} (L_s^{(k)} \otimes Q) \\
				& + c^{(k)} \gamma^{(k)} (\hat{D}^{(k)} \otimes Q) \Big] e(t).
			\end{aligned}
		\end{equation}

		According to the matrix inequality condition in the theorem \eqref{eqn:matrix_inequality}, we have:
		\begin{equation}
			\begin{aligned}
				\label{eqn:matrix_condition_applied}
				2 \sum_{k=1}^K \Big[ & C^{(k)} \gamma^{(k)} (L_s^{(k)} \otimes Q) \\
				& + c^{(k)} \gamma^{(k)} (\hat{D}^{(k)} \otimes Q) \Big] \geq \delta (I_N \otimes I_m).
			\end{aligned}
		\end{equation}

		Substituting this into \eqref{eqn:lyapunov_derivative_simplified}, we obtain:
		\begin{equation}
			\label{eqn:lyapunov_derivative_final}
			\dot{V}(t) \leq -\delta e^\mathrm{T}(t) e(t) = -\delta \| e(t) \|^2.
		\end{equation}
		
		Since $V(t)$ is positive definite and $\dot{V}(t)$ is negative definite, by LaSalle's invariance principle, the system state $e(t)$ will converge to the largest invariant set where $\dot{V}(t) = 0$, which implies $e(t) = 0$.
		
		Therefore, under the conditions of Theorem \ref{theorem:multi_pinning_consensus_modified}, the system \eqref{eqn:multi_dynamics} asymptotically achieves pinning consensus for all $i \in V$, i.e.,
		\begin{equation}
			\label{eqn:consensus_result_concluded}
			\lim_{t \to \infty} \| x_i(t) - x_i^* \| = 0, \quad \forall i \in V.
		\end{equation}
		
		This completes the proof.
	\end{proof}
	
	\begin{remark}
		\label{remark:control_gain_computation}
		The control gains \( c^{(k)} \) are determined by solving the Linear Matrix Inequality (LMI) presented in equation \eqref{eqn:matrix_inequality}. This LMI ensures that the system satisfies the required stability conditions for achieving pinning consensus. By formulating the problem as an LMI, we leverage convex optimization techniques to efficiently compute the optimal control gains.
	\end{remark}
	
	\section{Formulation of the Optimization Problem}
	\label{sec:optimization}
	
	In vehicular ad-hoc networks (VANETs), whether in a single network or a heterogeneous multi-network system, it is crucial to minimize the number of pinned vehicles while ensuring system stability. This problem is known to be NP-hard due to the combinatorial nature of selecting an optimal subset of pinning nodes. As the problem size grows, exact solutions become computationally infeasible, making heuristic optimization methods, such as genetic algorithms, a practical approach to efficiently identify a minimal set of pinning nodes that can achieve system convergence.
	
	Minimizing the number of pinned nodes reduces the required control resources, such as computational power and communication bandwidth, resulting in more cost-effective and scalable vehicular network systems. Efficient pinning node selection ensures that the control strategy remains effective as the number of vehicles in the network increases, facilitating the deployment of large-scale vehicular applications. By optimally selecting pinning nodes, the network can maintain stability and consensus even in the presence of node failures or communication disturbances, enhancing the overall robustness and reliability of the system. Finally, the optimization framework has broad applicability in real-world vehicular scenarios, including platooning, coordinated lane changes, adaptive cruise control, and traffic management systems, where synchronized behavior is critical for safety and efficiency. This formulation aims to achieve efficient control using the minimum number of control nodes while guaranteeing the stability of the overall system.
	
	\subsection{Constraints Analysis}
	
	To ensure that the vehicular system can achieve pinning consensus with the selected pinning nodes in both single and multi-network cases, the following constraints must be satisfied:
	
	\subsubsection{Stability Constraint}
	
	For the \textbf{single network case}, according to Theorem \ref{theorem:pinning_consensus}, the following linear matrix inequality (LMI) must hold to guarantee system stability and achieve pinning consensus:
	
	\begin{equation}
		\label{eqn:LMI_single}
		C L_s \otimes Q + c \hat{D} \otimes Q \geq \eta (I_N \otimes I_m),
	\end{equation}
	where:
	\begin{itemize}
		\item \( L_s = \frac{L + L^\mathrm{T}}{2} \) is the symmetric part of the Laplacian matrix \( L \).
		\item \( Q \succ 0 \) is a positive definite diagonal matrix.
		\item \( \hat{D} = \operatorname{diag}(d_1, d_2, \dots, d_N) \) is the pinning control matrix indicating which nodes are pinned.
		\item \( C > 0 \) is the coupling strength.
		\item \( c > 0 \) is the pinning control gain constant.
		\item \( \eta > 0 \) is a given constant ensuring the strictness of the inequality.
		\item \( I_N \) and \( I_m \) are identity matrices of size \( N \times N \) and \( m \times m \), respectively.
	\end{itemize}
	
	For the \textbf{heterogeneous multi-network case}, according to Theorem \ref{theorem:multi_pinning_consensus_modified}, the following LMI must hold:
	\begin{equation}
		\begin{aligned}
			\label{eqn:LMI_multi}
			C^{(k)} \gamma^{(k)} (L_s^{(k)} \otimes Q) & + c^{(k)} \gamma^{(k)} (\hat{D}^{(k)} \otimes Q) \\
			& \geq \eta (I_N \otimes I_m), \quad \forall k = 1, 2, \dots, K.
		\end{aligned}
	\end{equation}
	where:
	\begin{itemize}
		\item \( K \) is the number of networks.
		\item \( L_s^{(k)} = \frac{L^{(k)} + [L^{(k)}]^\mathrm{T}}{2} \) is the symmetric part of the Laplacian matrix \( L^{(k)} \) of network \( k \).
		\item \( \gamma^{(k)} > 0 \) is the internal coupling coefficient of network \( k \).
		\item \( C^{(k)} > 0 \) is the coupling strength in network \( k \).
		\item \( c^{(k)} > 0 \) is the pinning control gain in network \( k \).
		\item \( \hat{D}^{(k)} = \operatorname{diag}(d_1^{(k)}, d_2^{(k)}, \dots, d_N^{(k)}) \) is the pinning control matrix of network \( k \).
		\item \( Q \succ 0 \) is a positive definite diagonal matrix, common to all networks.
		\item \( \eta > 0 \) is a given constant.
	\end{itemize}
	
	This constraint ensures that the combined influence of all networks and their respective control gains is sufficient to stabilize the system by guaranteeing the negative definiteness of the Lyapunov function derivative, thus achieving pinning consensus across all networks.
	
	\subsubsection{Binary Variable Constraint}
	
	In both cases, the pinning control indicators \( d_i \) and \( d_i^{(k)} \) are binary variables:
	
	For the single network case:
	\begin{equation}
		d_i \in \{0, 1\}, \quad i = 1, 2, \dots, N.
	\end{equation}
	
	For the multi-network case:
	\begin{equation}
		d_i^{(k)} \in \{0, 1\}, \quad i = 1, 2, \dots, N; \quad k = 1, 2, \dots, K.
	\end{equation}
	
	This constraint reflects that each vehicle is either pinned or not in each network, which introduces complexity due to the combinatorial nature of the problem, especially in the multi-network case.
	
	\subsubsection{Positive Definiteness of Weighting Matrix}
	
	The weighting matrix \( Q \) must be positive definite:
	\begin{equation}
		Q \succ 0.
	\end{equation}
	
	By simplifying \( Q \) to a scalar, the LMI complexity is reduced, facilitating computational feasibility without compromising control efficacy across the vehicular networks. This ensures that the Lyapunov function is positive definite, which is necessary for the stability analysis in both single and multi-network cases.
	
	\subsubsection{Control Gain Constraint}
	
	The control gains must be positive to exert a stabilizing effect on the system:
	\begin{equation}
		c > 0, \quad \text{and} \quad c^{(k)} > 0, \quad k = 1, 2, \dots, K.
	\end{equation}
	Positive control gains amplify the influence of the pinned nodes, ensuring that the pinning control is effective in achieving system stability. This adaptability is essential for handling variations in network topologies and pinning configurations across different vehicular networks.
	
	\subsubsection{Challenges Arising from Constraint Conditions}	
	The feasibility of satisfying the constraints directly impacts whether the heuristic optimization algorithm can produce a viable solution. If the number of selected pinning nodes is too small, the LMI may have no solution, making it impossible to guarantee system stability. Conversely, increasing the number of pinning nodes could ensure that the LMI constraints are met and stability is achieved. However, this approach also significantly increases the load on the computational platform and communication costs due to direct interactions with the RSUs and the cloud platform.
	
	Thus, the optimization problem is not merely about determining the number of pinning nodes but also about selecting which specific vehicles should serve as pinning nodes. This selection must simultaneously satisfy the LMI constraint to ensure stability, adding complexity to the optimization process. While a multi-objective optimization approach could be considered to address these competing demands, this paper does not adopt a multi-objective model. The reason is that solutions from a multi-objective Pareto front do not inherently provide a single, actionable decision point. Instead, we formulate the optimization problem by treating the LMI feasibility as a constraint, thereby focusing on a solution that directly satisfies the system stability requirement.
	
	\subsection{Optimization Problem}
	
	\subsubsection{Single Network Case}
	
	In the single network scenario, the optimization problem aims to minimize the number of pinned nodes while satisfying the constraints:
	
	\begin{align}
		\text{Minimize} \quad & \sum_{i=1}^N d_i \\
		\text{Subject to} \quad & C L_s \otimes Q + c \hat{D} \otimes Q \geq \eta (I_N \otimes I_m), \\
		& d_i \in \{0, 1\}, \quad i = 1, 2, \dots, N, \\
		& Q \succ 0, \\
		& c > 0.
	\end{align}
	
	\subsubsection{Multi-Network Case}
	
	In the heterogeneous multi-network vehicular system, the optimization problem becomes more complex due to the presence of multiple networks and overlapping nodes. The objective is to minimize the total number of uniquely pinned nodes across all networks, which inherently encourages the selection of overlapping nodes as pinning nodes to enhance control efficiency. Selecting overlapping nodes as pinning nodes reduces redundant control actions across networks, thus enhancing resource efficiency. The optimization problem can be formulated as:
	
	\begin{align}
		\text{Minimize} \quad & \sum_{i=1}^N \mathcal{D}_i \\
		\text{Subject to} \quad 
		& \mathcal{D}_i = \max_{k} \{ d_i^{(k)} \}, \\
		& C^{(k)} \gamma^{(k)} (L_s^{(k)} \otimes Q) + c^{(k)} \gamma^{(k)} (\hat{D}^{(k)} \otimes Q) \nonumber \\ 
		& \quad \quad \quad \quad \quad \quad \quad \quad \geq \eta (I_N \otimes I_m), \\
		& \quad \quad \quad \quad \forall k = 1, \dots, K,  \nonumber\\
		& d_i^{(k)} \in \{0, 1\}, \quad i = 1, \dots, N, \\
		& Q \succ 0, \\
		& c^{(k)} > 0.
	\end{align}
	
	In this optimization problem:
	\begin{itemize}
		\item \( \mathcal{D}_i = 1 \) if vehicle \( i \) is pinned in at least one network, and \( \mathcal{D}_i = 0 \) otherwise. This prevents overlapping nodes from being counted multiple times.
		\item \( d_i^{(k)} = 1 \) if vehicle \( i \) is pinned in network \( k \), and \( d_i^{(k)} = 0 \) otherwise.
		\item The objective function \( \sum_{i=1}^N \mathcal{D}_i \) minimizes the total number of unique pinned nodes across all networks, which naturally promotes the selection of overlapping nodes without explicit weighting.
		\item The LMI constraint ensures that each network \( k \) achieves stability and consensus, given the control gains and pinning configuration.
	\end{itemize}
	
	\subsection{Challenges and Solution Strategy}
	
	The formulated optimization problems are characterized as mixed-integer linear matrix inequality (MILMI) problems, which are challenging to solve due to the combination of binary variables \( d_i \) and \( d_i^{(k)} \), representing pinning decisions, and the presence of LMI constraints, which ensure system stability. In large-scale vehicular networks, especially in the multi-network case, efficiently addressing this complexity is essential.
	
	In this context, the control gain variables \( c \), \( c^{(k)} \), and \( Q \) play a critical role in achieving the required system stability. Specifically:
	
	\begin{itemize}
		\item Control Gains \( c \) and \( c^{(k)} \): The gains \( c \) (for a single network) and \( c^{(k)} \) (for multiple networks) are dynamically determined through the solution of LMIs. For each network, an LMI is formulated based on the network’s Laplacian matrix and the pinning configuration. The solution of these LMIs provides feasible values for \( c \) and \( c^{(k)} \), ensuring the stability of each network. This adaptability is essential for handling variations in network topologies and pinning configurations across different vehicular networks.
		
		\item Weighting Matrix \( Q \): In the LMI formulation, \( Q \) is treated as a positive definite matrix variable. For computational simplicity, \( Q \) is often set as a scalar in the practical implementation, reducing the LMI complexity. This scalar \( Q \) is determined through the LMI solution process and plays a role in scaling the control influence across the network nodes.
	\end{itemize}
	
	By treating the LMI constraints as feasibility conditions, the optimization framework focuses on identifying a minimal set of pinning nodes that can satisfy system stability requirements across all networks while dynamically adjusting the control gains \( c \), \( c^{(k)} \), and \( Q \) to accommodate the varying network structures and pinning configurations.
	
	\section{Mixed-Variable Genetic Algorithm for Pinning Control Optimization}
	\label{sec:ga}
	
	This section presents an adaptive genetic algorithm (GA) tailored for pinning control optimization, including specific strategies for constraint handling and initialization, particularly under a heterogeneous multi-network vehicular network setup.
	
	\subsection{Penalty Function Design for Infeasibility of LMI}
	
	To ensure solutions meet the necessary Linear Matrix Inequality (LMI) constraints, a penalty function is incorporated within the GA, guiding the algorithm toward feasible regions. If the LMI constraint is violated, a penalty term \( \lambda \cdot \xi \) is added to the fitness function, where \( \lambda \) is the penalty coefficient. The penalty coefficient \( \lambda \) is dynamically adjusted throughout generations based on the infeasibility measure \( \xi \). Increasing \( \lambda \) for infeasible solutions with larger \( \xi \) encourages the population to move towards feasible regions, thus promoting convergence.
	
	For the \textbf{single-network case}, the infeasibility measure \( \xi_s \) quantifies the deviation from the LMI constraint:
	\[
	\xi_s = \left\| 2C L_s \otimes Q + 2c \hat{D} \otimes Q - \delta (I_N \otimes I_m) \right\|_F,
	\]
	where \( \| \cdot \|_F \) is the Frobenius norm. The fitness function \( fit_s \) for each individual is then defined as:	
	\begin{equation}
		\label{eqn:fitness_single}
		fit_s = \sum_{i=1}^N d_i + \lambda \cdot \xi_s,
	\end{equation}
	aiming to minimize the number of pinned nodes while penalizing infeasible solutions.
	
	For the \textbf{multi-network case}, the penalty function considers LMI constraint satisfaction across all networks. Let \( \xi_k \) denote the infeasibility measure for network \( k \):
	
	\[
	\begin{aligned}
		\xi_k = \Big\| 2C^{(k)} \gamma^{(k)} L_s^{(k)} \otimes Q^{(k)} &+ 2c^{(k)} \gamma^{(k)} \hat{D}^{(k)} \otimes Q^{(k)} \\
		&- \delta (I_N \otimes I_m) \Big\|_F.
	\end{aligned}
	\]	
	The overall infeasibility measure \( \xi_m \) aggregates across all networks:	
	\[
	\xi_m = \sum_{k=1}^K \xi_k.
	\]	
	The fitness function \( fit_m \) for each individual is defined as:
	
	\begin{equation}
		\label{eqn:fitness_multi}
		fit_m = \sum_{i=1}^N \mathcal{D}_i + \lambda \cdot \xi_m,
	\end{equation}	
	where \( \mathcal{D}_i = \max_k d_i^{(k)} \) indicates whether node \( i \) is pinned in at least one network. This function promotes solutions that minimize the total number of unique pinned nodes while ensuring system-wide LMI feasibility.
	
	\subsection{Initialization and Representation of \( \hat{D} \) and \( d_i \)}
	
	The objective in both single and multi-network cases is to minimize the total number of pinned nodes. In the \textbf{single-network case}, each binary variable \( d_i \) is defined for node \( i \) as:	
	\begin{equation}
		\hat{D} = \operatorname{diag}(d_1, d_2, \dots, d_N), \quad d_i \in \{0, 1\},
	\end{equation}	
	where \( d_i = 1 \) if node \( i \) is pinned and \( d_i = 0 \) otherwise. The optimization objective reduces to minimizing \( \sum_{i=1}^N d_i \).
	
	In the \textbf{multi-network case}, binary variables \( d_i^{(k)} \) represent the pinning state of node \( i \) in network \( k \):
	
	\begin{equation}
		\hat{D}^{(k)} = \operatorname{diag}(d_1^{(k)}, d_2^{(k)}, \dots, d_N^{(k)}), \quad d_i^{(k)} \in \{0, 1\},
	\end{equation}	
	for \( i = 1, 2, \dots, N \) and \( k = 1, 2, \dots, K \). The variable \( \mathcal{D}_i \) captures whether node \( i \) is pinned in at least one network:	
	\begin{equation}
		\mathcal{D}_i = \max_{k} \{d_i^{(k)}\}, \quad \mathcal{D}_i \in \{0, 1\}, \quad i = 1, 2, \dots, N.
	\end{equation}
	The optimization goal becomes:	
	\[
	\sum_{i=1}^N \mathcal{D}_i.
	\]	
	Each \( d_i^{(k)} \) is initially set to 1 with a low probability \( p_d \), promoting sparsity in pinning node selection:	
	\[
	d_i^{(k)} = \begin{cases}
		1 & \text{with probability } p_d, \\
		0 & \text{with probability } 1 - p_d.
	\end{cases}
	\]	
	This setup supports independent pinning across networks while accurately accounting for overlapping nodes through \( \mathcal{D}_i \).
	
	\subsection{Initialization of System Parameters \( C \), \( L_s \), and \( Q \)}
	
	For each network \( k \), the coupling strengths \( C^{(k)} \) are initialized as constants, assuming uniform interaction strength across networks:	
	\[
	C^{(k)} = C_k,
	\]	
	where \( C_k \) is predefined for network \( k \). This simplifies parameter initialization and reduces computational complexity. The initial values of coupling strengths \( C^{(k)} \) are selected within a specified range to ensure balanced control influence across networks, which is particularly important for heterogeneous networks with varying topologies and requirements.
	
	Each network’s Laplacian matrix \( L^{(k)} \) is derived from the adjacency matrix \( G^{(k)} \), with the symmetric part \( L_s^{(k)} \) calculated as:	
	\[
	L_s^{(k)} = \frac{L^{(k)} + [L^{(k)}]^T}{2}.
	\]	
	The positive definite matrix \( Q^{(k)} \) is determined during the LMI solving process for each network \( k \). For computational simplicity, \( Q^{(k)} \) can be treated as a scalar or a diagonal matrix.

	\subsection{Adaptive Genetic Selection Strategy}
	
	To solve the optimization problem involving binary decision variables and LMI constraints, we propose an adaptive genetic algorithm, given in Algorithm \ref{alg:ga}, specifically designed to address the unique challenges of pinning control in vehicular networks. The algorithm focuses on optimizing the binary variables \( d_i \) (for the single network case) or \( d_i^{(k)} \) (for the multi-network case) representing the pinning decisions.
	
	\begin{algorithm}[!htp]
		\caption{Mixed-Variable Genetic Algorithm for Pinning Control Optimization}
		\begin{algorithmic}[1]
			\REQUIRE System parameters: adjacency matrices \( G^{(k)} \), initial coupling strengths \( C^{(k)} \); GA parameters: population size \( N_{\text{pop}} \), maximum generations \( N_{\text{gen}} \), crossover probability \( p_c \), mutation probability \( p_m \), penalty coefficient \( \lambda \).
			
			\ENSURE Optimal pinning node set \( d^* \).
			
			\STATE \textbf{Population Initialization:} Generate an initial population of \( N_{\text{pop}} \) individuals. Each individual includes binary variables \( d_i \) (for single network) or \( d_i^{(k)} \) (for multi-network) representing pinning nodes.
			
			\FOR{each individual in the population}
			\STATE Initialize \( d_i \) or \( d_i^{(k)} \in \{0, 1\} \) with a low probability \( p_d \) to encourage sparsity.
			\STATE Compute \( \mathcal{D}_i = \max_k d_i^{(k)} \) (for multi-network case).
			\STATE For the given pinning configuration, solve the LMI(s) to find feasible \( c \), \( c^{(k)} \), and \( Q^{(k)} \).
			\IF{feasible solution for \( c \) and \( Q^{(k)} \) is found}
			\STATE Set \( \xi = 0 \) and calculate the fitness using \eqref{eqn:fitness_single} or \eqref{eqn:fitness_multi}.
			\ELSE 
			\STATE Calculate the infeasibility measure \( \xi \) and compute the fitness with the penalty term.
			\ENDIF
			\ENDFOR
			
			\FOR{generation \( \text{gen} = 1 \) to \( N_{\text{gen}} \)}
			\STATE Select individuals based on fitness using tournament selection.
			
			\STATE Perform crossover on \( d_i \) or \( d_i^{(k)} \) with probability \( p_c \).
			
			\STATE Apply mutation to offspring with probability \( p_m \) by flipping bits in \( d_i \) or \( d_i^{(k)} \).
			
			\STATE For each offspring, compute $\mathcal{D}_i$ (if applicable) and solve the LMI(s) to verify feasibility.
			
			\STATE Calculate fitness for each offspring using the updated pinning configuration and penalize infeasible solutions.
			
			\STATE Combine parents and offspring, select the best \( N_{\text{pop}} \) individuals based on fitness.
			\ENDFOR
			
			\STATE \textbf{Output:} The optimal pinning node set \( d^* \).
		\end{algorithmic}
		\label{alg:ga}
	\end{algorithm}
	The feasibility of the LMI solution is directly related to the selection of pinning nodes. In extreme cases where all vehicle nodes are chosen as pinning nodes, the LMI is more likely to have a feasible solution, but this negates the purpose of optimizing pinning control to reduce communication overhead. The proposed GA offers a viable approach to balance LMI feasibility and minimizing the number of pinned nodes.
	
	By focusing on the binary pinning variables and determining \( c \), \( c^{(k)} \), and \( Q^{(k)} \) through LMI solving during fitness evaluation, the algorithm efficiently navigates the search space. The penalty function allows the GA to retain high-quality solutions that may initially violate constraints, preserving promising individuals in the search for an optimal pinning configuration.
	
	The adaptive genetic selection strategy ensures diversity in the population and promotes convergence towards feasible and optimal solutions. By incorporating constraint handling directly into the fitness evaluation, the GA effectively addresses the mixed-integer nature of the optimization problem involving both combinatorial decisions and continuous feasibility constraints.

		\section{Simulation and Discussion}
	\label{sec:sim}
	
	\subsection{Scenario Description and Simulation Settings}
	
	This study considers two main scenarios: a single vehicular network and a multi-network vehicular ad-hoc network (VANET) with overlapping sub-networks. Each vehicle is modeled as a node in a structured traffic environment, represented as a first-order continuous-time multi-agent system.
	
	\subsubsection{Single Network Scenario}
	
	In the single network scenario, vehicles are tasked with maintaining stable traffic flow by tracking a lead vehicle, with only a subset of vehicles designated as “pinned vehicles” receiving direct control inputs. This selective pinning approach minimizes control requirements across the network. The network is represented by an asymmetric adjacency matrix \( G \), where \( G_{ij} \) denotes the presence and strength of a communication link from vehicle \( j \) to vehicle \( i \), capturing inter-vehicular influences. We use a genetic algorithm (GA) to identify the optimal set of control nodes, aiming to minimize interventions while ensuring system-wide stability.
	
	\subsubsection{Multi-Network Scenario}
	
	In the multi-network scenario, we analyze a heterogeneous VANET composed of multiple overlapping sub-networks. Each sub-network represents a unique communication topology with distinct nodes and connection patterns. Due to network overlap, some vehicles belong to multiple sub-networks, resulting in shared nodes that reflect real-world communication and cooperation requirements. Each sub-network is described by an asymmetric adjacency matrix \( G^{(k)} \) for \( k = 1, 2, \dots, K \), where \( K = 3 \) in this study. The GA-based pinning control strategy is applied independently within each sub-network to select optimal control nodes, with overlapping nodes prioritized to maximize control efficiency.
	
	\subsection{Simulation Settings for Single Vehicular Network}
	
	The target state for all vehicles, \( x_{\text{desired}} \), was set to 90. Initial vehicle states were randomly initialized following a Gaussian distribution with mean \( \mu = 100 \) and standard deviation \( \sigma = 15 \), simulating diverse starting conditions. The adjacency matrix \( G \) was designed as a sparse, asymmetric matrix, where elements \( g_{ij} \) were generated randomly within \( [0, 1] \) as follows:
	\[
	G_{ij} = 
	\begin{cases} 
		g_{ij}, & \text{if } g_{ij} > 0.5, \\
		0, & \text{if } g_{ij} \leq 0.5.
	\end{cases}
	\]
	This configuration yields a sparse matrix with \( G_{ij} > 0 \) denoting a directed link from node \( i \) to node \( j \). The coupling strength \( C \) was set to 0.8, and the control gain \( c \) was set within the range \( c \in [0, 50] \).
	
	\subsection{Simulation Settings for Multi-Network Vehicular System}
	
	The multi-network VANET comprises overlapping sub-networks, each with distinct topologies. Vehicles within each sub-network aim to reach specific target states, such as speed or position, through interactions with neighboring nodes. Shared nodes across networks represent vehicles participating in multiple sub-networks, introducing variability in communication and cooperation. 
	
	Each network is represented by an asymmetric adjacency matrix \( G^{(k)} \) for \( k = 1, 2, \ldots, K \) (where \( K = 3 \) here), capturing directional information flow. Adjacency matrix elements \( g_{ij} \) were generated as follows:
	\[
	G_{ij}^{(k)} = 
	\begin{cases} 
		g_{ij}, & \text{if } g_{ij} > 0.8, \\
		0, & \text{if } g_{ij} \leq 0.8,
	\end{cases}
	\]
	where \( g_{ij} \) represents the communication strength between nodes \( i \) and \( j \), sampled from a uniform distribution. Due to the relatively large number of vehicles in the multi-network vehicular system, we adjusted the threshold value to 0.8 to make the coupling matrix sparser, aligning it more closely with real-world conditions.
	
	To evaluate the control strategy’s scalability, we tested three scenarios with vehicle counts of 50, 100, and 200, corresponding to small, medium, and large network scales. Each sub-network controls an independent state variable, with target values set to 50, 70, and 120. Initial values of each state variable were randomly generated from normal distributions with means of 45, 80, and 130, and variances of 10, 12, and 8, respectively. Vehicles were assigned either exclusively to a single network or concurrently to multiple networks.
	
	\subsection{Configuration for Adaptive Genetic Algorithm}
	
	For the single network, the GA was configured with \( N_{\text{pop}} = 100 \), \( N_{\text{gen}} = 20 \), a crossover probability \( p_c = 0.8 \), and a mutation probability \( p_m = 0.05 \). The penalty coefficient \( \lambda = 10 \) in the fitness function prioritized minimizing pinned nodes while ensuring stability.
	
	In the multi-network vehicular system, the GA’s population size was scaled based on network size, with values set to 100, 200, and 400 for vehicle counts of 50, 100, and 200, respectively. All other GA parameters were consistent with the single network configuration.
	
	\subsection{Results and Discussions in Single Vehicular Network}
	
	Figure \ref{fig:convergence_plot} illustrates the state trajectories of all vehicles as they converge towards the desired state. The solid red lines represent the pinned vehicles, which were directly controlled by the system, while the dashed blue lines indicate the unpinned vehicles, whose states were indirectly influenced by the pinned nodes through network coupling. The target state is shown as a black dashed line for reference. As observed, both the pinned and unpinned nodes successfully converged to the desired state within approximately \( T = 5 \) seconds of simulation time. The pinned nodes, being directly controlled, exhibited faster convergence compared to the unpinned nodes, which followed more gradual trajectories.
	
	The final error between the system state and the desired state at the end of the simulation was effectively zero for all 50 vehicles, demonstrating the effectiveness of the proposed control strategy in maintaining system stability. Furthermore, the genetic algorithm selected about 22 pinned nodes (around 40\% of all nodes), which is a relatively small subset of the total \( N = 50 \) vehicles. This result confirms that the pinning control strategy, when optimized by GA, can achieve consensus with minimal control effort, reducing both computational and communication overhead in large-scale vehicular networks.
	
	\begin{figure}[!htp]
		\centering
		\includegraphics[width=0.8\linewidth]{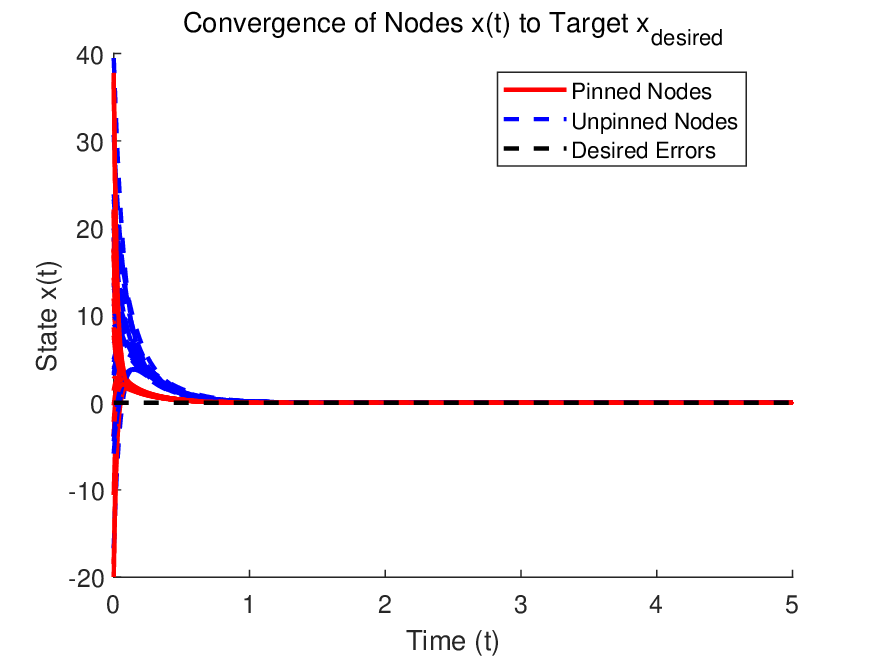}
		\caption{State trajectories of all vehicles converging to the desired state. Pinned vehicles are shown with solid red lines, unpinned vehicles with dashed blue lines, and the desired error is indicated by the black dashed line.}
		\label{fig:convergence_plot}
	\end{figure}
	
	As shown in Fig. \ref{fig:convergence_plot}, the states of the pinned nodes, controlled through cloud platforms or RSUs, reach the control target in approximately 0.3 seconds. The non-pinned nodes, influenced by the pinned nodes, achieve the target state within 1 second, effectively meeting the requirements of vehicular network applications.
	
	The choice of control gain \( c \) has a significant impact on the number of vehicles designated as pinning nodes. When the value of \( c \) is not appropriately chosen, stability can still be achieved by increasing the number of pinned nodes. To further investigate this relationship, we conducted experiments using fixed values of \( c \) rather than determining it through the LMI optimization. In 50 trials with varying coupling conditions, we set control gain values of 1, 2, 3, 4, and 5, and performed statistical analysis. The results are illustrated in Fig.~\ref{fig:number} and Fig.~\ref{fig:error}. 
	
	As shown in Fig.~\ref{fig:number}, the average number of pinned vehicles for the five fixed control gain values across 50 trials were 37, 40, 42, 43, and 41 for control gains of 1, 2, 3, 4, and 5, respectively. This indicates that, on average, 80\% of the vehicles need to be pinned, which limits the efficiency of pinning control and increases the communication load on roadside units or cloud platforms. Additionally, as illustrated in Fig.~\ref{fig:error}, the errors associated with fixed control gains are significantly higher than those achieved with the LMI-based control gain strategy, underscoring the advantages of the LMI approach in reducing the number of control nodes required for stability.
	
	\begin{figure}[!htp]
		\centering
		\includegraphics[width=0.8\linewidth]{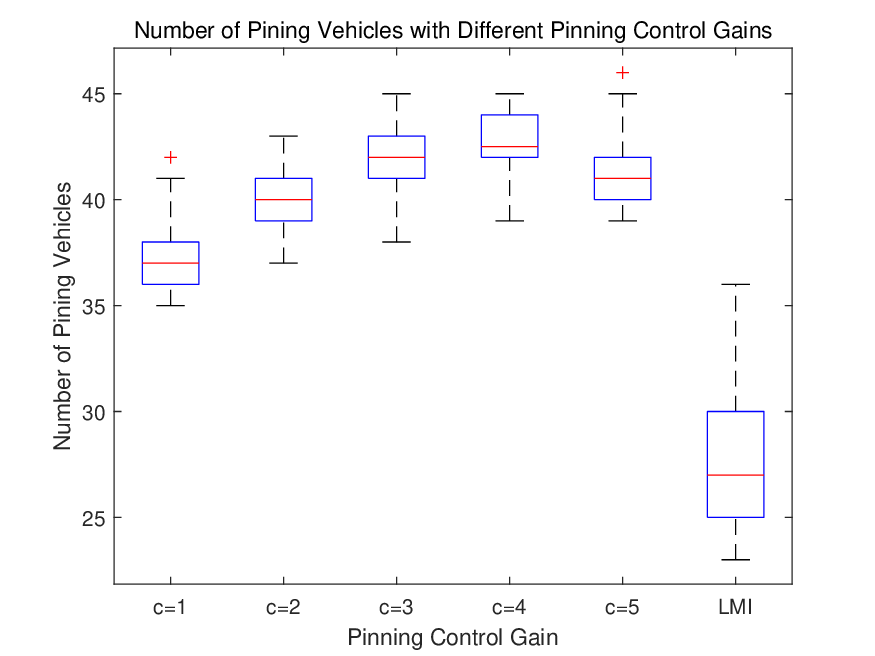}
		\caption{Statistical analysis of the number of pinned vehicles for control gains \( c = 1, 2, 3, 4, 5 \) in 50 trials of asymmetric coupling random scenarios, compared to the control gain selected by the LMI method.}
		\label{fig:number}
	\end{figure}
	
	\begin{figure}[!htp]
		\centering
		\includegraphics[width=0.8\linewidth]{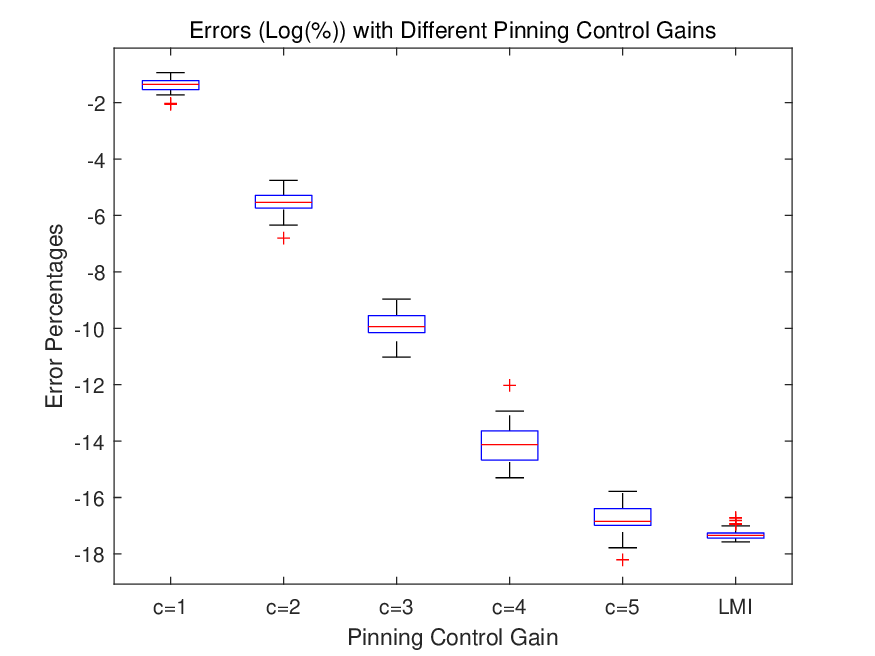}
		\caption{Statistical analysis of pinning control consistency errors for control gains \( c = 1, 2, 3, 4, 5 \) in 50 trials of asymmetric coupling random scenarios, compared to the control gain selected by the LMI method.}		
		\label{fig:error}
	\end{figure}
	
	Overall, the LMI-based control gain strategy demonstrates superior performance, achieving stability with fewer control nodes and lower error rates. This highlights the LMI method's effectiveness and efficiency in large-scale vehicular network applications, providing robust control with optimized resource allocation.

	\subsection{Results and Discussion for Multi-Network Vehicular Systems}
	
	The simulation results for the multi-network vehicular system are presented below. Figure \ref{fig:50selection} shows the distribution of 50 vehicles across three networks, with the x-axis representing vehicle IDs, the y-axis denoting network IDs, and an additional line indicating Pinning Status. In this setup, Network 1 contains 35 vehicles, Network 2 has 25 vehicles, and Network 3 also has 25 vehicles. Of the 50 vehicles, 25 are present in two networks simultaneously, and 5 vehicles participate in all three networks. Vehicles in Network 1 are marked in blue, those in Network 2 in red, and those in Network 3 in green.
	
	Some vehicles exist solely within a single network; for instance, vehicle ID 1 exists only in Network 1, while vehicle ID 16 is shared between Network 1 and Network 2, and vehicle ID 21 spans all three networks. If a vehicle is selected as a pinning node, its Pinning Status is indicated by a black diamond marker. For example, as shown, vehicle ID 1 is selected as a pinning node, while vehicle ID 3 is not. As depicted in Fig.~\ref{fig:50selection}, 33 out of 50 vehicles (approximately 60\%) are selected as pinning nodes. Of these, 4 vehicles belong to all three networks (80\% of such vehicles) and 18 out of the 25 vehicles (72\%) that exist in two networks are chosen as pinning nodes. This demonstrates the importance of selecting shared vehicles in multi-network settings to enhance pinning control efficiency. As illustrated in Fig.~\ref{fig:50con}, the control strategies across the three networks reach consensus within approximately 1 to 1.5 seconds.
	
	\begin{figure*}[!htbp]
		\centering
		\includegraphics[width=2.2\columnwidth]{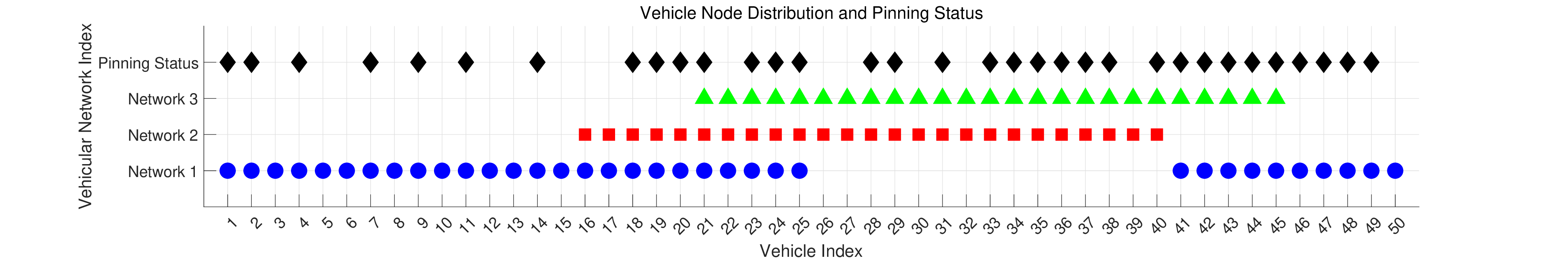}
		\caption{Distribution of 50 Vehicles Across Three Networks and Selection Scheme for Pinning Vehicle Nodes}
		\label{fig:50selection}
	\end{figure*}
	
	\begin{figure}[!htbp]
		\centering
		\includegraphics[width=\columnwidth]{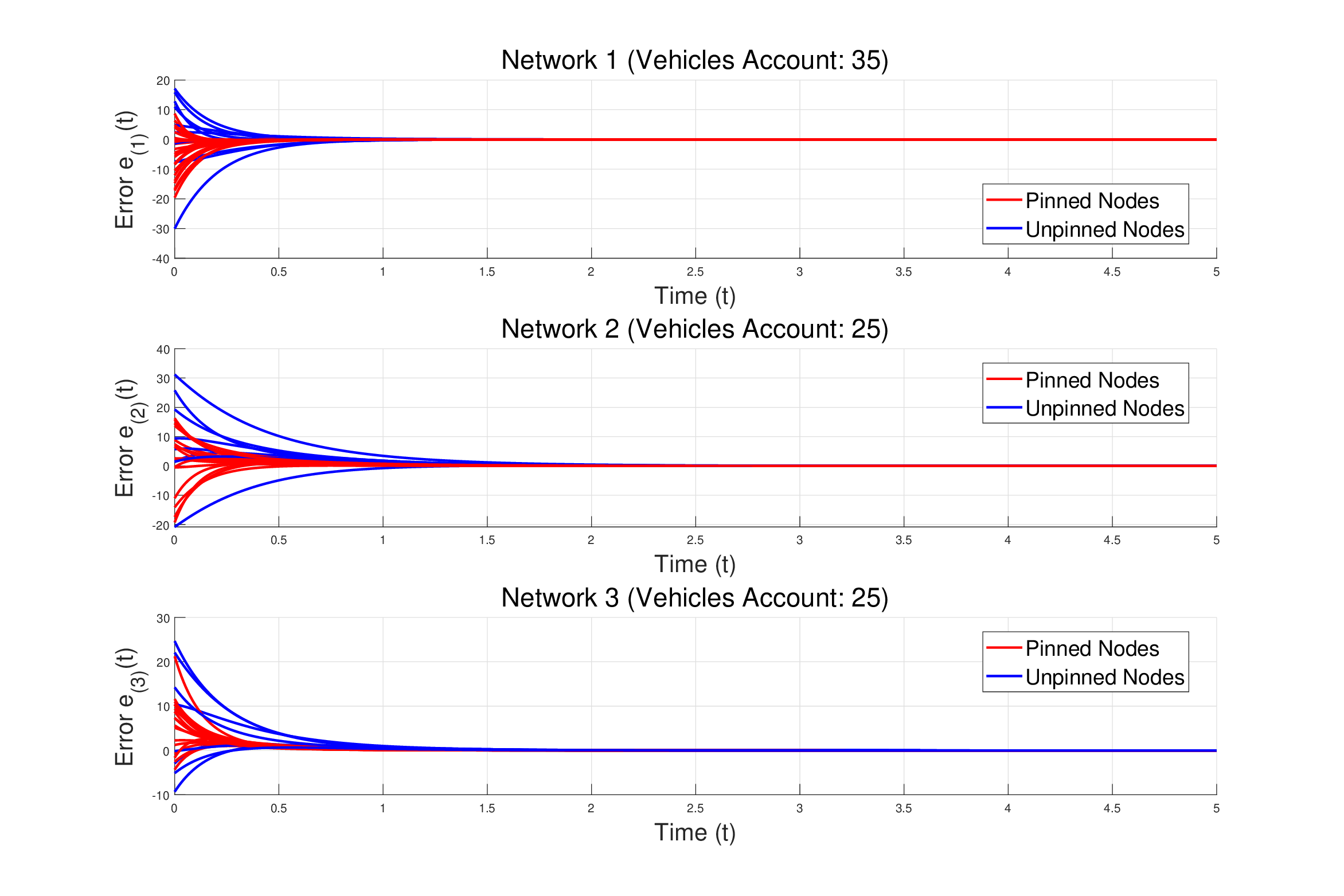}
		\caption{Convergence of 50 Vehicle states across multi-network vehicular systems. Each sub-network's convergence behavior is illustrated separately, indicating both the pinned and unpinned vehicles' convergence trajectories. }
		\label{fig:50con}
	\end{figure}
	
	For the medium-scale scenario with 100 vehicles, as shown in Fig.~\ref{fig:100selection}, Network 1 contains 70 vehicles, while Networks 2 and 3 each contain 50 vehicles. Of these vehicles, 10 are shared across all three networks, and 50 are shared between two networks, with the remaining 40 confined to a single network. A total of 77 vehicles (77\% of the total) were selected as pinning nodes. Specifically, 8 of the 10 vehicles that span all three networks (80\%) were pinned, along with 39 of the 50 vehicles (78\%) shared by two networks and 30 of the 40 vehicles (75\%) within a single network. As shown in Fig.~\ref{fig:100con}, consensus among pinning control across multiple states is achieved within approximately 1 to 1.5 seconds for this 100-vehicle scenario.
	
	\begin{figure*}[!htbp]
		\centering
		\includegraphics[width=2.2\columnwidth]{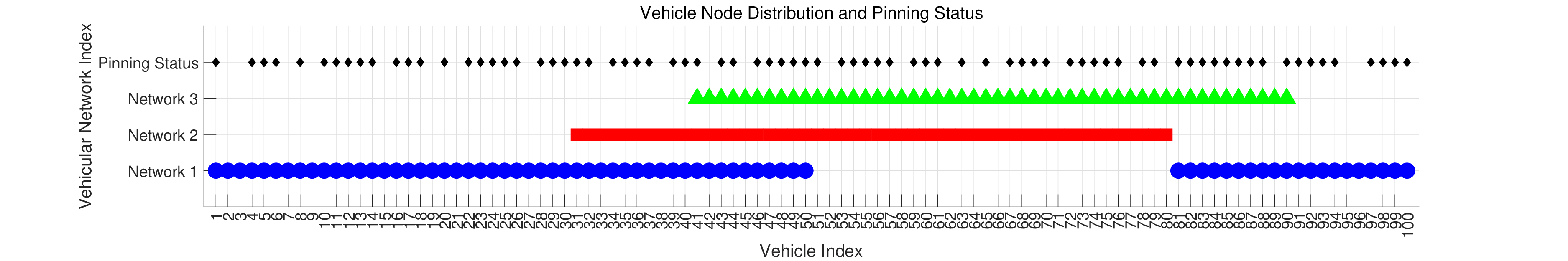}
		\caption{Distribution of 100 Vehicles Across Three Networks and Selection Scheme for Pinning Vehicle Nodes.}
		\label{fig:100selection}
	\end{figure*}
	
	\begin{figure}[!htbp]
		\centering
		\includegraphics[width=\columnwidth]{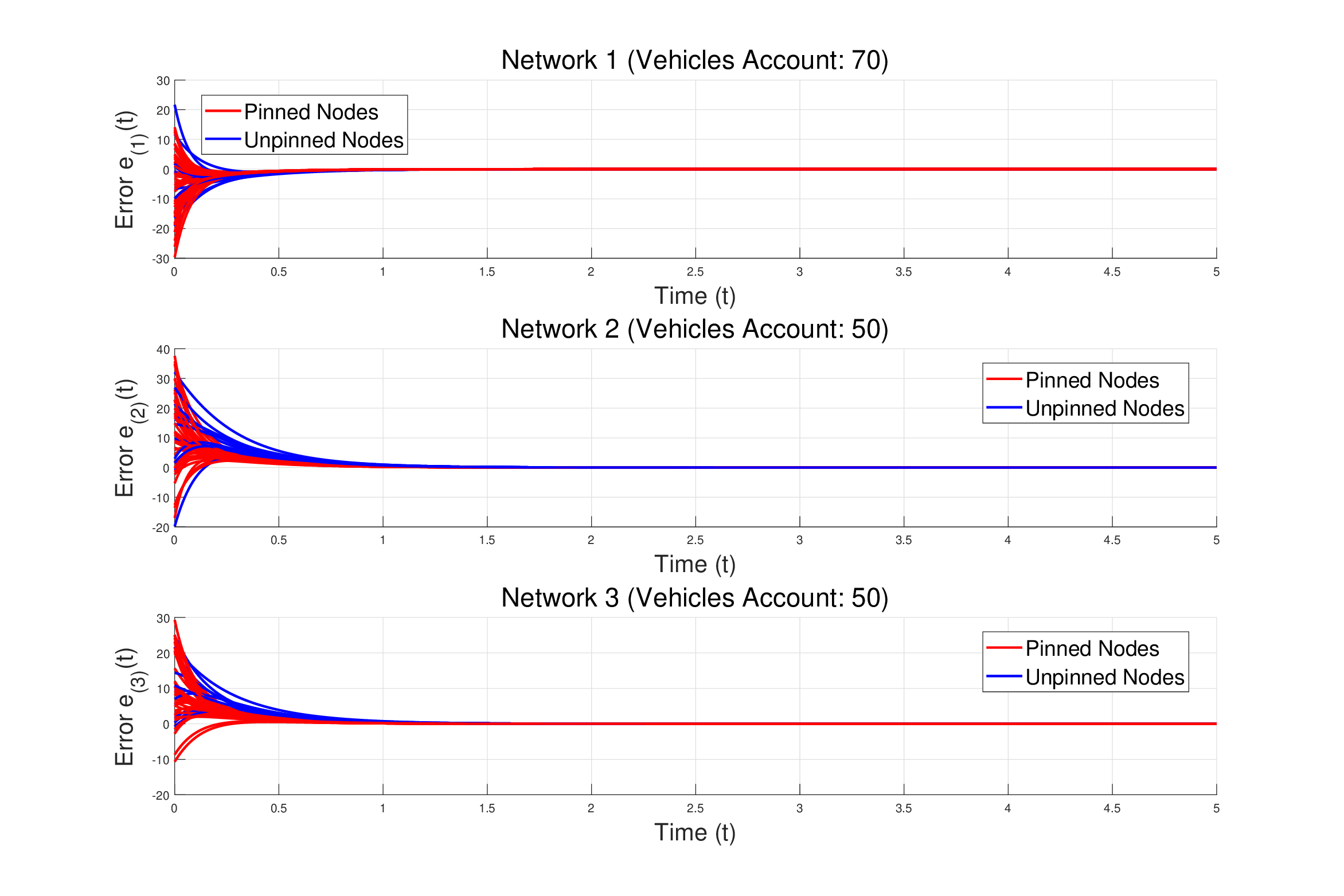}
		\caption{Convergence of 100 Vehicle states across multi-network vehicular systems. Each sub-network's convergence behavior is illustrated separately, indicating both the pinned and unpinned vehicles' convergence trajectories.}
		\label{fig:100con}
	\end{figure}
	
	In the large-scale scenario with 200 vehicles, illustrated in Fig.~\ref{fig:200selection}, Network 1 contains 140 vehicles, and Networks 2 and 3 each contain 100 vehicles. Among these, 20 vehicles exist in all three networks, 100 are shared between two networks, and 80 are confined to a single network. Based on the algorithm's selection, 129 vehicles (64\% of the total) are chosen as pinning nodes. Specifically, 13 of the 20 vehicles shared across all three networks (65\%) were pinned, along with 69 of the 100 vehicles (69\%) shared by two networks, and 47 of the 80 vehicles (58\%) within a single network. As shown in Fig.~\ref{fig:200con}, the control performance across the 200 vehicles achieves consensus within approximately 1 to 1.5 seconds.
	
	\begin{figure*}[!htbp]
		\centering
		\includegraphics[width=2.2\columnwidth]{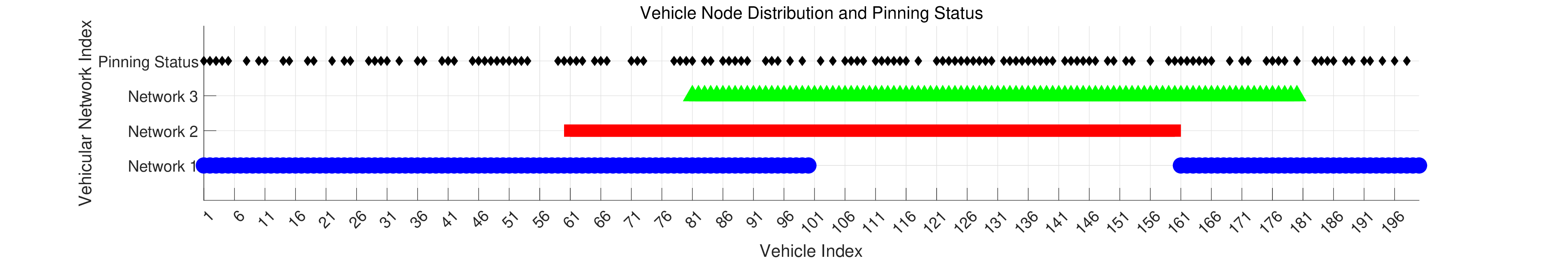}
		\caption{Distribution of 200 Vehicles Across Three Networks and Selection Scheme for Pinning Vehicle Nodes}
		\label{fig:200selection}
	\end{figure*}
	
	\begin{figure}[!htbp]
		\centering
		\includegraphics[width=\columnwidth]{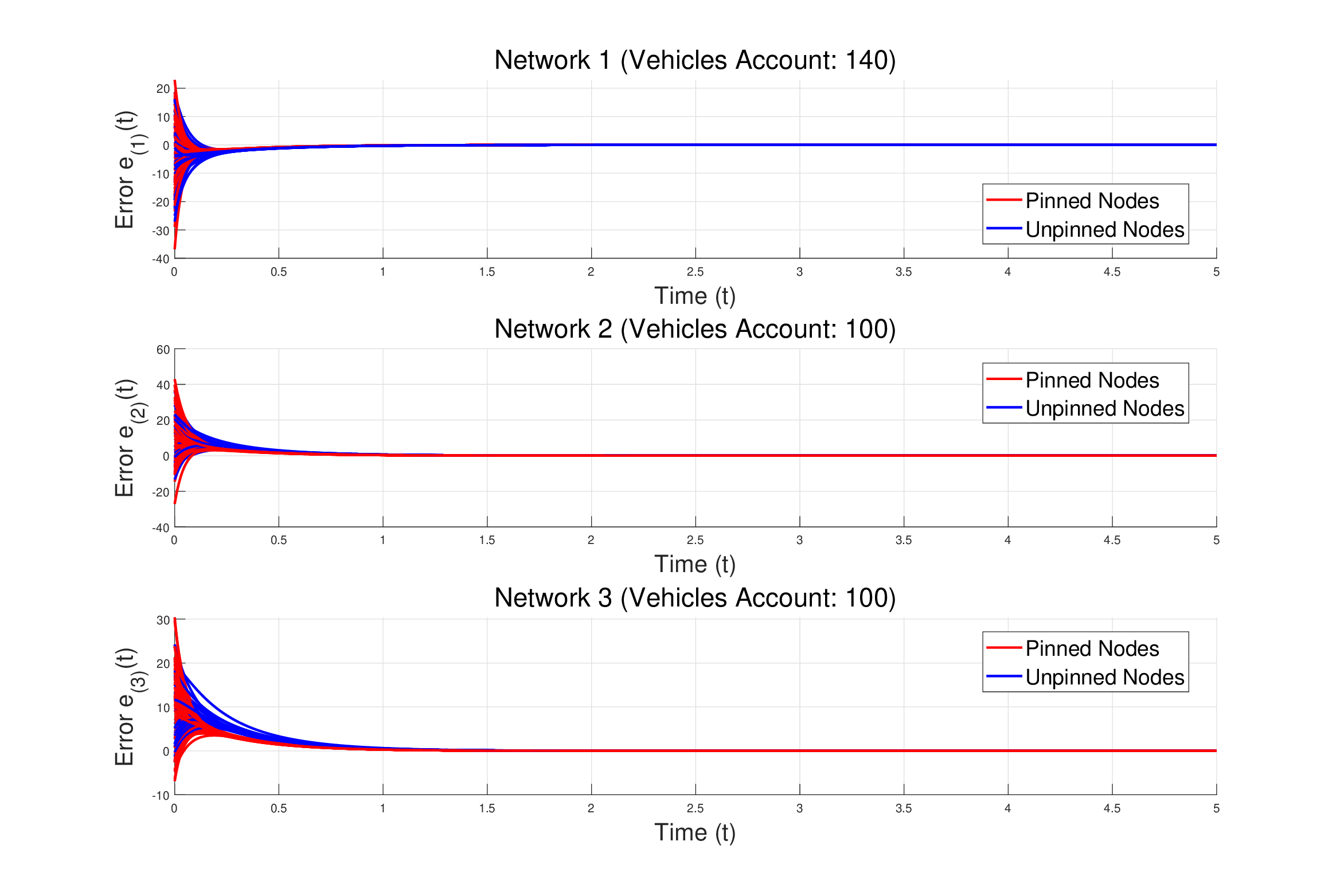}
		\caption{Convergence of 200 Vehicle states across multi-network vehicular systems. Each sub-network's convergence behavior is illustrated separately, indicating both the pinned and unpinned vehicles' convergence trajectories.}
		\label{fig:200con}
	\end{figure}
	
	For a comprehensive assessment, we conducted a statistical analysis involving 30 independent experiments for each of the 50-, 100-, and 200-vehicle cases. This analysis, illustrated in Fig.~\ref{fig:vq_per} and Fig.~\ref{fig:vq_error}, evaluates the proportion of pinned vehicles and control error. As the number of vehicles increases, the proportion of pinned vehicles stabilizes at a relatively low level, effectively enhancing the control efficiency of the network. In contrast, with fewer vehicles, inter-network connections vary more substantially, resulting in greater fluctuations in the proportion of pinned vehicles. As shown in Fig.~\ref{fig:vq_error}, across scenarios with 50, 100, and 200 vehicles, the control error converges to a range between \(10^{-7}\) and \(10^{-10}\) within 1.5 seconds, indicating the robustness of the control algorithm and the effectiveness of the pinning node selection strategy.
	
	\begin{figure}[!htbp]
		\centering
		\includegraphics[width=0.8\columnwidth]{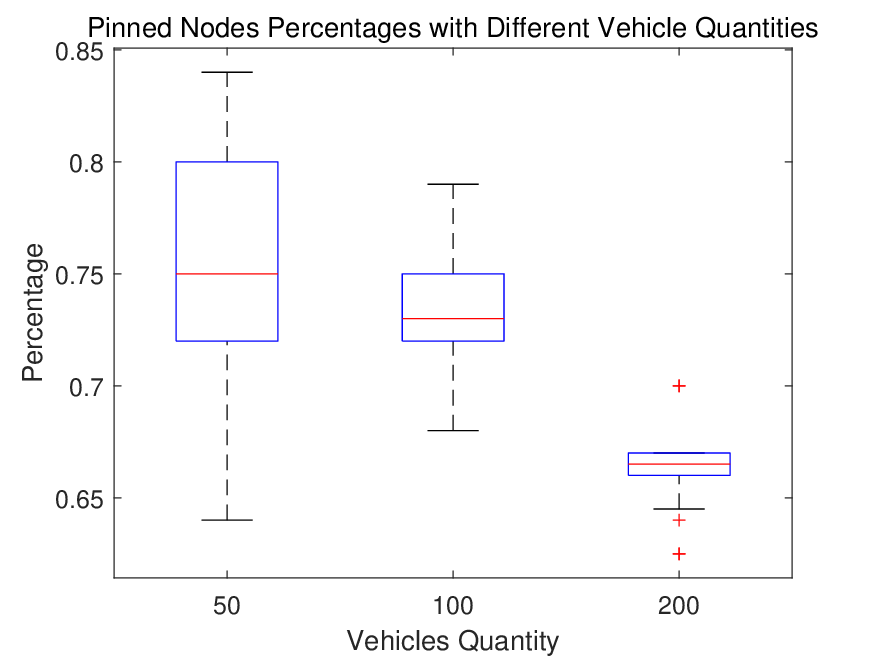}
		\caption{Statistical Analysis for Percentage of Pinned Vehicles under Pinning Control with 50, 100, and 200 Vehicles.}
		\label{fig:vq_per}
	\end{figure}
	
	\begin{figure}[!htbp]
		\centering
		\includegraphics[width=0.8\columnwidth]{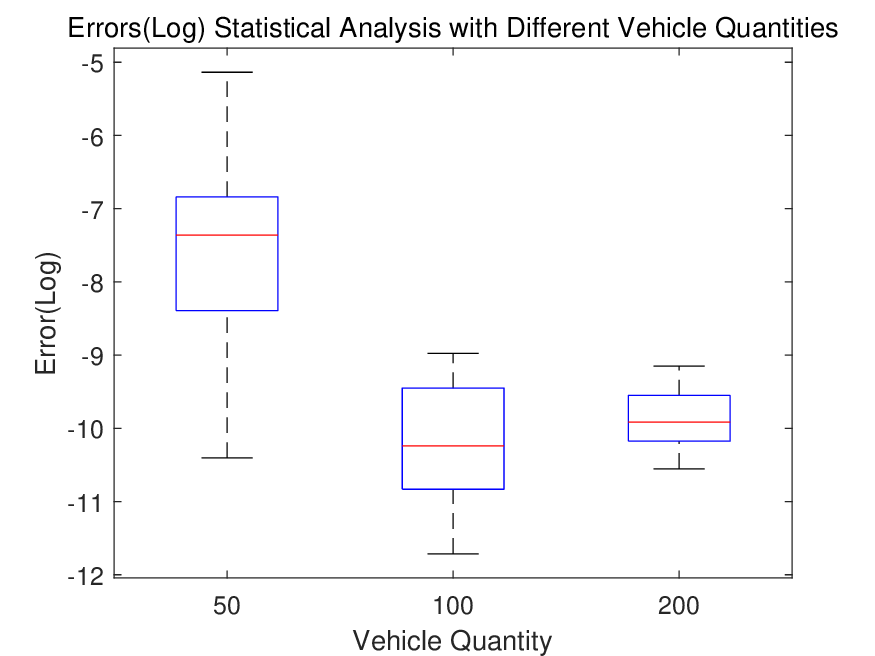}
		\caption{Statistical Analysis for Error (Log) of Pinned Vehicles under Pinning Control with 50, 100, and 200 Vehicles.}
		\label{fig:vq_error}
	\end{figure}

	\section{Conclusions and Future Work}
	\label{sec:con}	
	In this paper, we have developed a comprehensive pinning control optimization framework for heterogeneous multi-network VANETs. The main contributions and findings can be summarized in three aspects. First, we established theoretical foundations for pinning control convergence in both single and multi-network scenarios, providing rigorous stability conditions through Lyapunov analysis and LMI formulation. Second, we proposed an adaptive mixed-variable genetic algorithm that effectively addresses the challenges of asymmetric coupling and network overlap, incorporating an adaptive penalty function to handle LMI constraints. Third, through extensive simulations across different network scales, we demonstrated that our framework significantly reduces the number of required control nodes while maintaining system stability, particularly when leveraging overlapping nodes across multiple networks.
	
	The simulation results validated several key advantages of our approach. The proposed method achieved rapid consensus across various network configurations while minimizing control overhead. The framework demonstrated robust performance in handling time delays and topology changes, essential characteristics for practical VANET implementations. Moreover, the prioritization of overlapping nodes as control candidates proved effective in reducing overall system complexity while maintaining control efficiency.
	
	Future research could explore further enhancements to the GA framework, focusing on reducing computational costs and improving scalability for larger networks. The method could also be extended to accommodate diverse vehicle types with different control objectives, particularly in mixed autonomous and human-driven scenarios. 

	\bibliographystyle{IEEEtran}  
	\bibliography{refe.bib}

	\vspace{0pt}
	\begin{IEEEbiography}[{\includegraphics[width=1in,height=1.25in,clip,keepaspectratio]{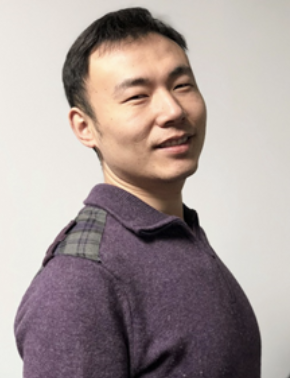}}]{Weian GUO} received the M.Eng. degree in navigation, guidance, and control from Northeastern University, Shenyang, China, in 2009, and the doctor of engineering degree from Tongji University, Shanghai, China, in 2014. From 2011 to 2013, he was sponsored by China Scholarship Council to carry on his research at the Social Robotics Laboratory, National University of Singapore. He is currently an associate Professor with the Sino-German College of Applied Science, Tongji University. His interests include computational intelligence, control theory and vehicular network.
	\end{IEEEbiography}
	
	\begin{IEEEbiography}[{\includegraphics[width=1in,height=1.25in,clip,keepaspectratio]{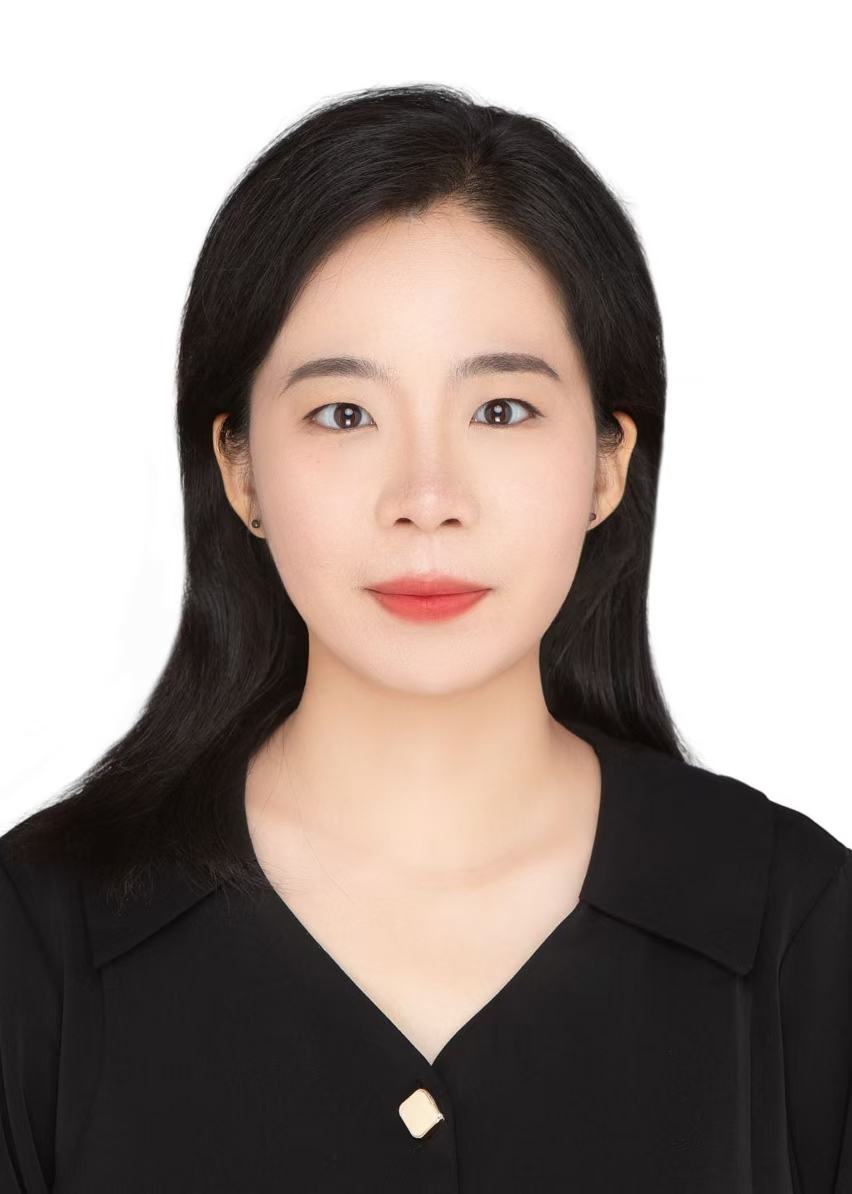}}]{Ruizhi SHA} received the B.S. degree in automa-
		tion in 2014, and the Ph.D. degree in control science and engineering in 2023, from Xi’an Jiaotong University, Xi’an,China. Since 2024, she has been working in the School of Electronic and Information Engineering, Tongji University. Her research interests include uncertain systems control, nonlinear systems control, and their applications.
	\end{IEEEbiography}

	\begin{IEEEbiography}[{\includegraphics[width=1in,height=1.25in,clip,keepaspectratio]{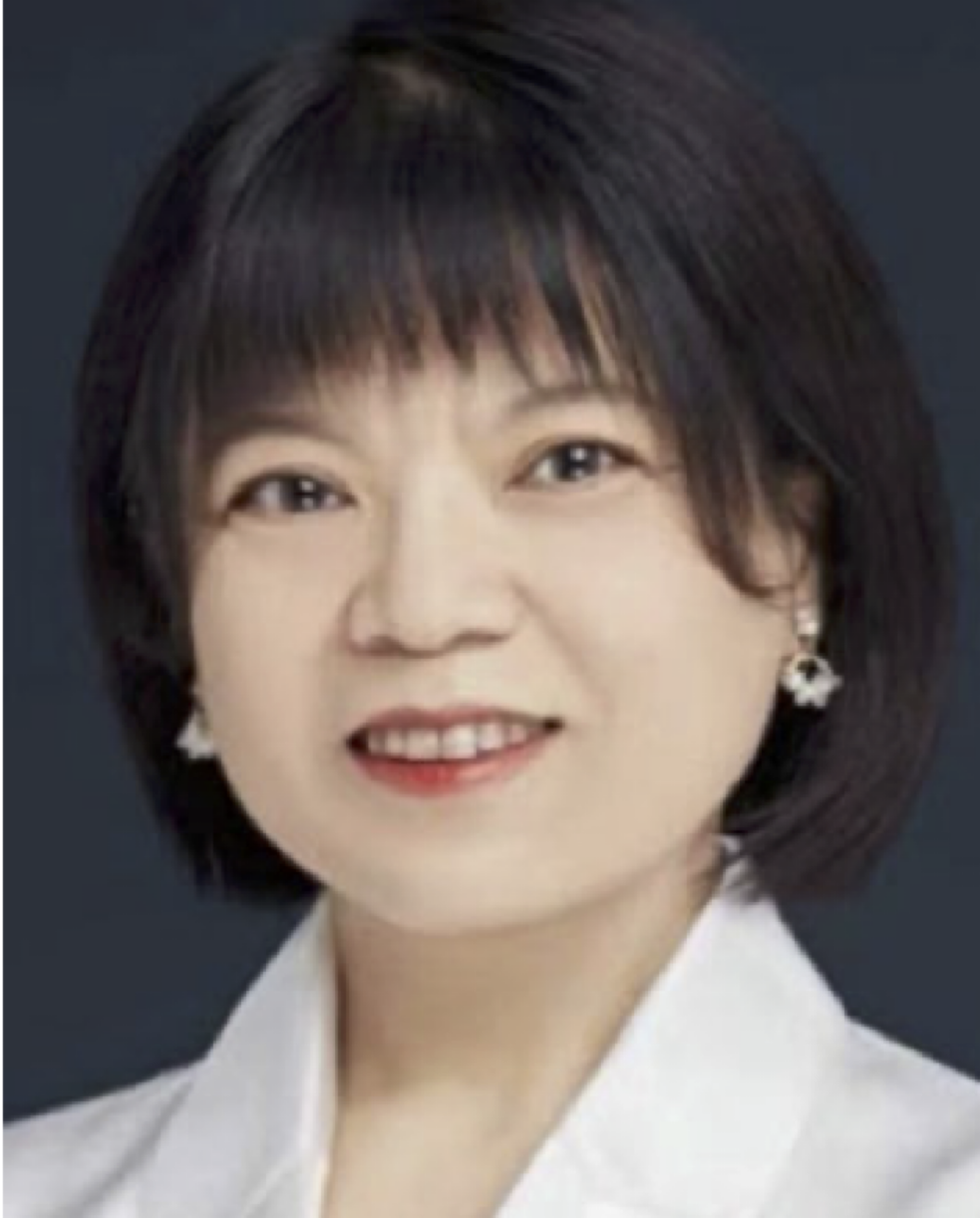}}]{Li LI} received the B.S. and M.S. degrees in electrical automation from Shengyang Agriculture University, Shengyang, China, in 1996 and 1999, respectively, and the Ph.D. degree in mechatronics engineering from the Shenyang Institute of Automation, Chinese Academy of Science, Shenyang, in 2003. She joined Tongji University, Shanghai, China, in 2003, where she is currently a professor of control science and engineering. She has over 50 publications, including 4 books, over 30 journal papers, and 2 book chapters. Her current research interests include production planning and scheduling, computational intelligence, data-driven modeling and optimization, semiconductor manufacturing, and energy systems.
	\end{IEEEbiography}
	\begin{IEEEbiography}[{\includegraphics[width=1in,height=1.25in,clip,keepaspectratio]{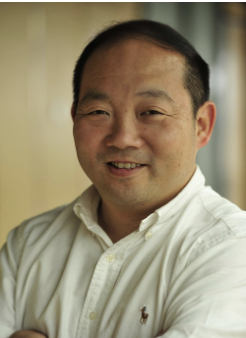}}]{Lun ZHANG}
		received the B.S. and Ph.D degrees in computer communications, transportation information engineering and control, Tongji University, Shanghai, China, in 1992 and 2005, respectively. He is currently a professor with School of Transportation, Tongji University. His interests include intelligent transportation, computational intelligence, and deep learning.
	\end{IEEEbiography}	
	\begin{IEEEbiography}[{\includegraphics[width=1in,height=1.25in,clip,keepaspectratio]{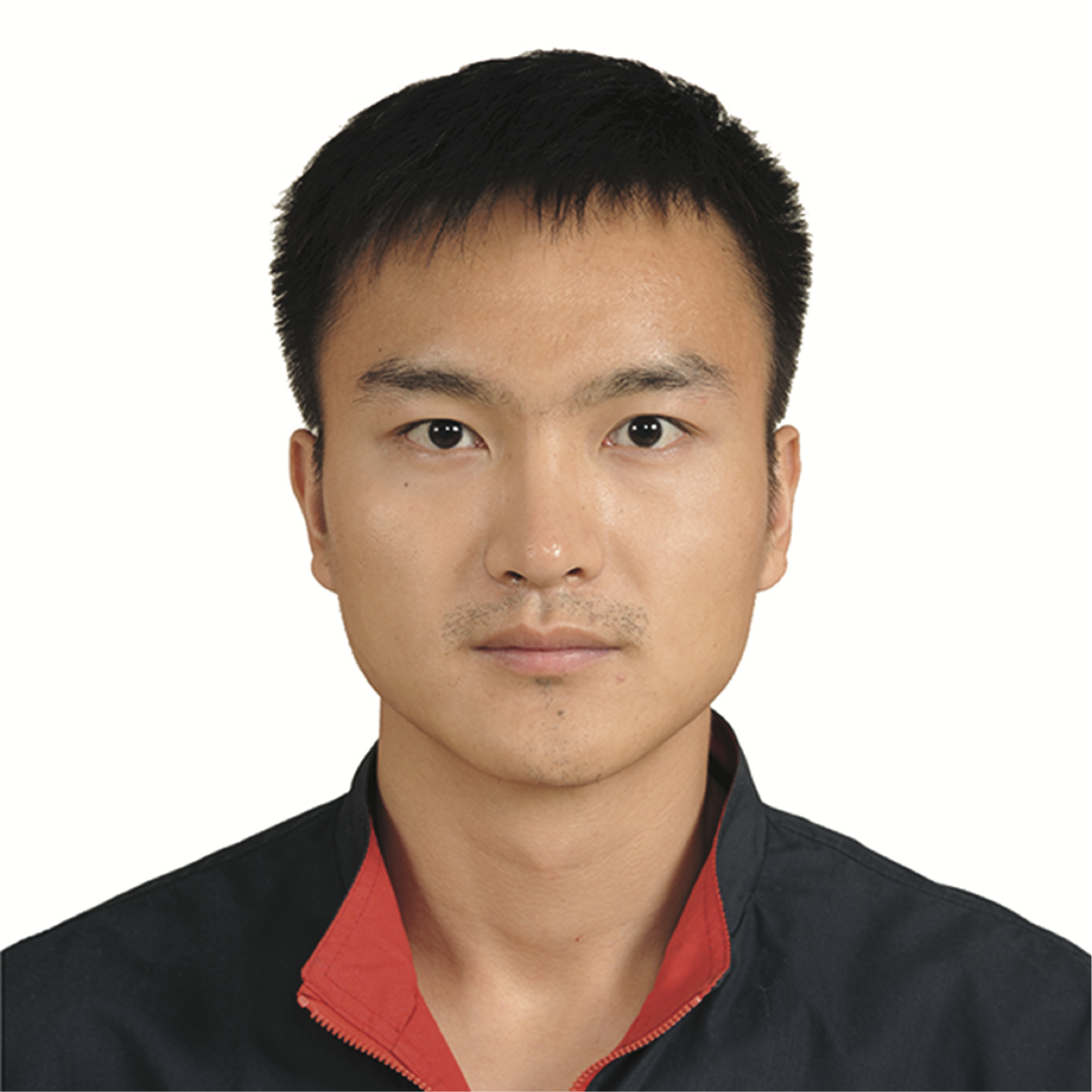}}]{Dongyang Li}
		received the M.S. and Ph.D degrees in school of electronics and information engineering, Tongji University, Shanghai, China, in 2017 and 2022, respectively. From 2019 to 2021, he was sponsored by China Scholarship Council to carry on his research at Georgia Institute of Technology. He is now an engineer with the Sino-German College of Applied Science, Tongji University; His research interest includes computational intelligence, deep learning and their applications.
	\end{IEEEbiography}
	
\end{document}